\documentclass[final,3p,times]{elsarticle}
\usepackage{amsmath,amssymb,amsfonts}
\usepackage{amsthm}
\usepackage{graphicx}
\usepackage{hyperref}
\usepackage{algc}
\newcommand{\Real}{\mathbb{R}}
\newcommand{\bfx}{\boldsymbol{x}}

\usepackage{xcolor}

\usepackage{amssymb,amsmath,graphicx}

\def\tr{{\raise0pt\hbox{$\scriptscriptstyle\top$}}}

\newtheorem{example}{Example}
\newtheorem{proposition}{Proposition}
\newtheorem{theorem}{Theorem}
\newtheorem{remark}{Remark}
\newtheorem{corollary}{Corollary}
\newtheorem{definition}{Definition}

\definecolor{redcolor}{rgb}{0.7,0.3,0.3}

\begin{document}

\begin{frontmatter}

\title{Correction of AI systems by linear discriminants: Probabilistic foundations}

\author[LeicMath,NNU]{A.N. Gorban\corref{cor1}}
\ead{a.n.gorban@le.ac.uk}
\author[LETI]{A. Golubkov}
\author[LeicMath]{B. Grechuk}
\ead{bg83@le.ac.uk}
\author[LeicMath,NNU]{E.M. Mirkes}
\ead{em322@le.ac.uk}
\author[LeicMath,NNU]{I.Y. Tyukin}
\ead{i.tyukin@le.ac.uk}

\address[LeicMath]{Department of Mathematics, University of Leicester, Leicester, LE1 7RH, UK}
\address[NNU]{Lobachevsky University, Nizhni Novgorod, Russia}
\address[LETI]{Saint-Petersburg State Electrotechnical University, Saint-Petersburg,  Russia}

\cortext[cor1]{Corresponding author}

\begin{abstract}
{Artificial Intelligence (AI) systems  sometimes make errors and will make errors in the future,  from time to time. These errors are usually unexpected, and can lead to dramatic consequences. }Intensive development of AI and its practical applications makes the problem of errors more important. Total re-engineering of the systems can create new errors and is not always possible due to the resources involved. The important challenge is to develop  fast methods to correct errors without damaging existing skills. We formulated the technical requirements to the `ideal' correctors. Such correctors include binary classifiers, which separate the situations with high risk of errors from the situations where the AI systems work properly.   Surprisingly, for essentially high-dimensional data such methods are possible:  simple linear Fisher discriminant can separate the situations with errors from correctly solved tasks even for exponentially large samples. The paper presents  the  probabilistic basis for fast  non-destructive correction of AI systems.  A series of new stochastic separation theorems is proven.  { These theorems provide new instruments for fast non-iterative correction of errors  of legacy AI systems. The new approaches become efficient in high-dimensions, for correction of high-dimensional systems in high-dimensional world (i.e. for processing of essentially high-dimensional data by large systems).}

We prove that this separability property holds for a wide class of distributions including log-concave distributions and distributions with a special `SMeared Absolute Continuity' (SmAC)  property defined through relations between the volume and probability of sets of vanishing volume. These classes are much wider than the Gaussian distributions. The requirement of independence and identical distribution of data is significantly relaxed. The results are supported by computational analysis of empirical data sets.
\end{abstract}


\end{frontmatter}

\section{Introduction}

\subsection{Errors and correctors of AI systems}

State-of-the art Artificial Intelligence (AI) systems for data mining consume huge and fast-growing collections of heterogeneous data. Multiple versions of these huge-size systems have been deployed to date on millions of computers and gadgets across many various platforms. Inherent uncertainties in data result in unavoidable mistakes (e.g. mislabelling, false alarms, misdetections, wrong predictions etc.) of the AI data mining systems, which require judicious use. These mistakes become gradually more important because of numerous and increasing number of real life AI applications in such sensitive areas as security, health care, autonomous vehicles and robots. Widely advertised success in testing of AIs in the laboratories often can  not be reproduced in realistic operational conditions. Just for example,  Metropolitan Police's facial recognition matches (`positives') are reported 98\% inaccurate (false positive), and South Wales Police's matches are reported 91\% inaccurate \cite{Foxx2018}. {Later on, experts in statistics mentioned that `figures showing inaccuracies of 98\% and 91\% are likely to be a misunderstanding of the statistics and are not verifiable from the data presented' \cite{face2018}. Nevertheless, the large number of false positive recognitions of `criminals' leads to serious concerns about security of AI use because people have to prove their innocence as police are wrongly identifying thousands of innocent citizens as criminals.}

The successful functioning of any AI system in realistic operational condition dictates that mistakes must be detected and corrected immediately and locally in the networks of collaborating  systems. Real-time correction of the mistakes by re-training is not always viable due to the resources involved. Moreover, the re-training could introduce new mistakes and damage existing skills.
All AI systems make errors. Correction of these errors is gradually becoming an increasingly important problem. 

The technical requirements to the `ideal' correctors can be formulated as follows \cite{GorTyukPhil2018}. Corrector should: (i)  be simple; (ii)  not damage the skills of the legacy system in the situations, where they are working successfully;  (iii) allow fast non-iterative learning; and (iv) allow correction of the new mistakes without destroying  the previous fixes.

\begin{figure}
\centering
\includegraphics[width=0.5\textwidth]{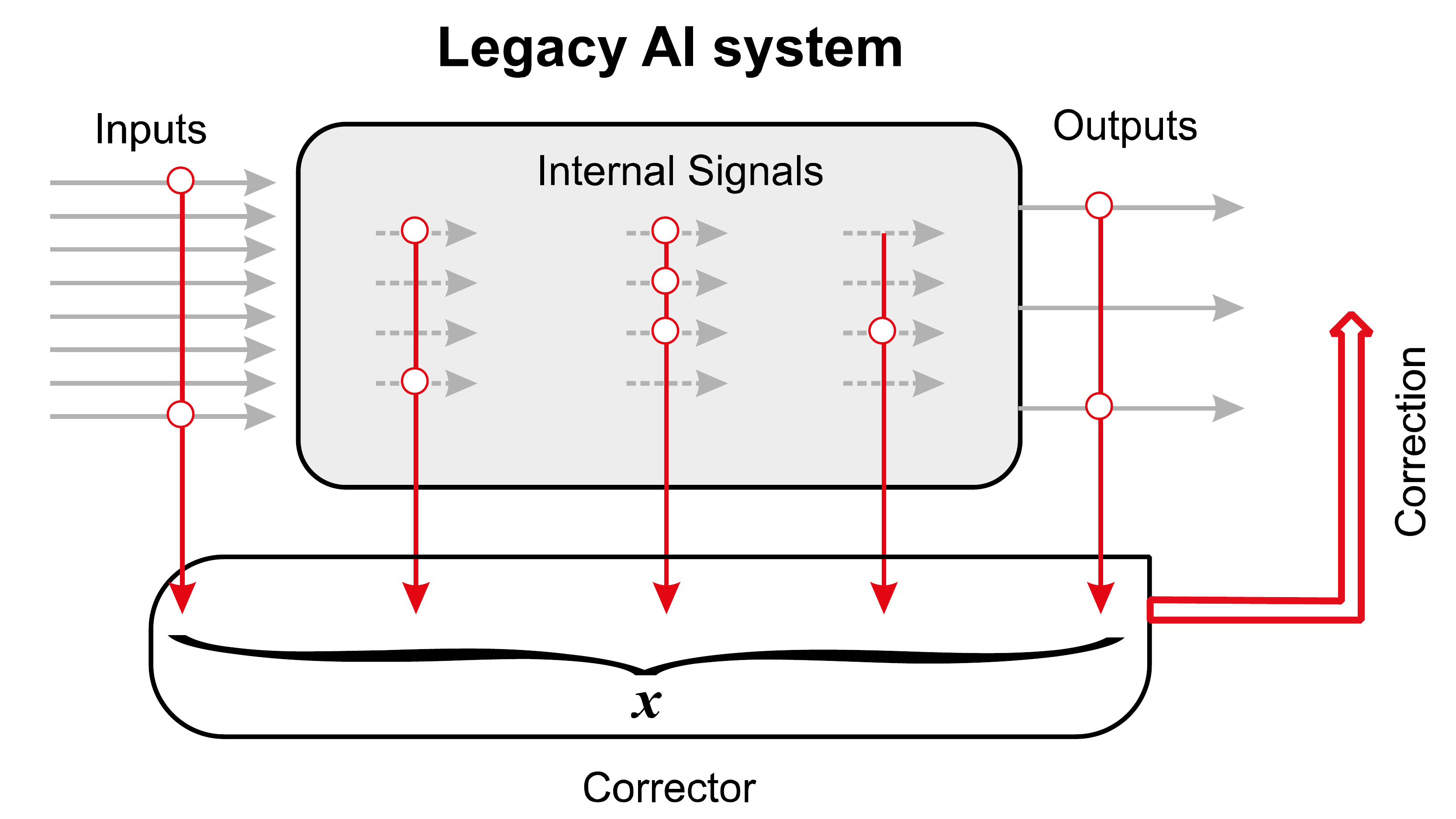}
\caption {Corrector of AI errors. Inputs for this corrector may include input signals, and any internal or output signal of the AI system (marked by circles).}
\label{Fig:Corrector}
\end{figure}

The recently proposed architecture of such a corrector is simple \cite{GorTyuRom2016}. It consists of two ideal devices:
\begin{itemize}
\item A binary classifier for separation of the situations with possible mistakes form the situations with correct functioning (or, more advanced, separation of  the situations with high risk of mistake from the situations with low risk of mistake);
\item A new decision rule for the situations with possible mistakes (or with high risk of mistakes).
\end{itemize}
A binary classifier is the main and universal part of the corrector for any AI system, independently of tasks it performs. The corrected decision rule is more specific.

Such   a corrector is an external system, and  the main legacy AI system remains unchanged (Fig.~\ref{Fig:Corrector}). One corrector can correct several errors (it is useful to cluster them before corrections). Cascades of correctors are employed for further correction of more errors \cite{GorTyukPhil2018}: the AI system with the first corrector is a new legacy AI system and can be corrected further (Fig.~\ref{Fig:AIcorrectorsCascade}).

\begin{figure}
\centering
\includegraphics[width=0.5\textwidth]{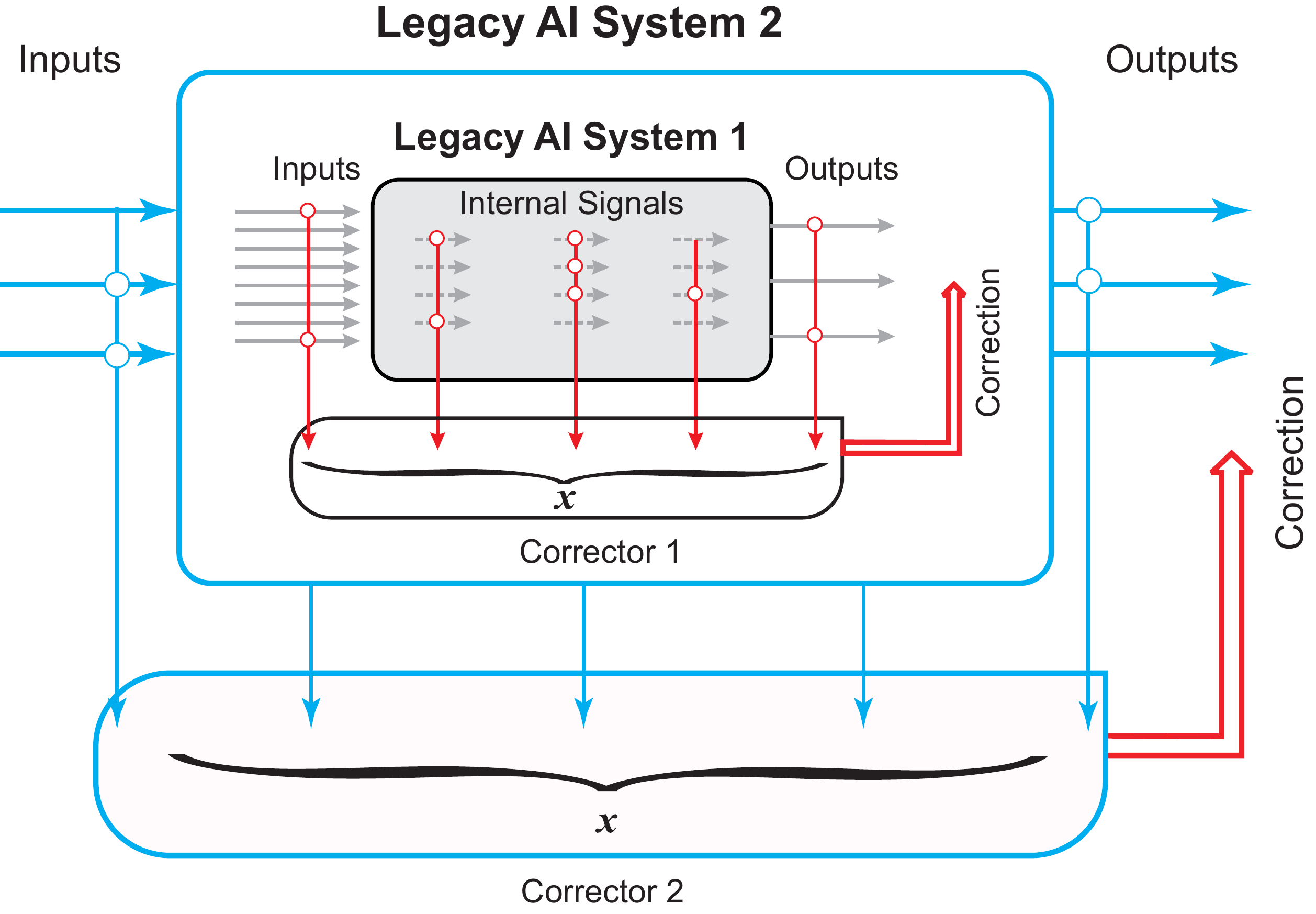}
\caption{Cascade of AI correctors. In this diagram, the original legacy AI system (shown as Legacy AI System 1) is supplied with a corrector altering its responses.  The combined new AI system can in turn be augmented by another corrector, leading to a cascade of AI correctors. }\label{Fig:AIcorrectorsCascade}
\end{figure}

\subsection{Correctors and blessing of dimensionality}

Surprisingly, if the dimension of the data is high enough, then the classification problem in corrector construction can be solved by simple linear Fisher's discriminants even if the data sets are exponentially large with respect to dimension. This phenomenon is the particular case of the {\em blessing of dimensionality}. This term was coined \cite{Kainen1997,Donoho2000} as an antonym of the `curse of dimensionality'  for the group of high-dimension geometric phenomena which simplify data mining in high dimensions.

Both curse and blessing of dimensionality are manifestations of the measure concentration phenomena, which were discovered in the foundation of statistical physics and analysed further in the context of geometry, functional analysis, and probability theory (reviewed by \cite{GiannopoulosMilman2000,GorTyukPhil2018,Ledoux2005}). The `sparsity' of high-dimensional spaces and  concentration of measure phenomena make some low-dimensional approaches impossible in high dimensions. This problem is widely known as the `curse of dimensionality' \cite{Trunk1979,Donoho2000,Pestov2013}.  The same phenomena can be efficiently employed for creation of new, high-dimensional methods, which seem to be much simpler in high dimensions  than the  low-dimensional approaches. This is the `blessing of dimensionality' \cite{Kainen1997,Donoho2000,AndersonEtAl2014,GorTyuRom2016,GorTyukPhil2018}.

Classical theorems about concentration of measure  state that random points in a high-dimensional data distribution are concentrated in a thin layer near an average or median level set of a Lipschitz function (for introduction into this area we refer to \cite{Ledoux2005}). The stochastic separation theorems \cite{GorbanRomBurtTyu2016,GorbTyu2017} revealed the fine structure of these thin layers: the random  points  are all linearly separable by simple Fisher's discriminants from the rest of the set even for exponentially large random sets. Of course, the probability distribution should be `genuinely' high-dimensional for all these concentration and separation theorems.

The correctors with higher abilities can be constructed on the basis of small neural networks with uncorrelated neurons \cite{GorbanRomBurtTyu2016} but already single-neuron correctors (Fisher's discriminants) can help in explanation of a wealth of empirical evidence related to in-vivo recordings of `Grandmother' cells and `concept' cells \cite{GorTyukPhil2018,TyukinBrain2017}.

{The theory and the concept have been tested in several case-studies of which the examples are provided in \cite{GorbanRomBurtTyu2016, Tyukin2017a}. In these works, the underlying use-case was the problem of improving performance of legacy AI systems built for pedestrian detection in video streams. In \cite{GorbanRomBurtTyu2016} we showed that spurious errors of a legacy Convolutional Neural Network with VGG-11 architecture could be learned away in a one-shot manner at a near-zero cost to exiting skills of the original AI \cite{PatentRomanenkoGorTyu}. In \cite{Tyukin2017a} we demonstrated that A) the process can be fully automated to allow an AI system (teacher AI) to teach another (student AI), and B) that the errors can be learned away in packages. The approach, i.e. Fisher discriminants, stays largely the same, albeit in the latter case cascaded pairs introduced in \cite{GorbanRomBurtTyu2016} were used instead of mere single hyperplanes.}

\subsection{The contents of this paper }

In this paper we aim to present  probabilistic foundations for correction of errors in AI systems in high dimensions. We develop a new series of stochastic separation theorems, prove some hypotheses formulated in our paper \cite{GorbTyu2017}, answer one general problem about stochastic separation published by Donoho and Tanner \cite{DonohoTanner2009}, and specify general families of probability distributions with the stochastic separation properties (the measures with `SMeared Absolute Continuity' property).

The first highly non-trivial question in the analysis of curse and blessing of dimensionality is: what is the dimensionality of data? What does it mean that `the dimension of the data is high enough'? Of course, this dimension does not coincide with the dimension of the data space and can be significantly lower. The appropriate definition of data dimension depends on the problem we need to solve. Here we would like to create linear or even Fisher's classifier for separation of mistakes from the areas, where the legacy AI system works properly. Therefore, the desirable evaluations of data dimension should characterise and account for the possibility to solve this problem with high accuracy and low probability of mistakes. We will return to the evaluation of this probability in all sections of the paper.

In Sec.~\ref{Sec:DiscrimPrelim} we define the notions of linear separability and Fisher's separability and   present a basic geometric construction whose further development allows us to prove Fisher's separability of high-dimensional data and to estimate the corresponding dimension of the data in the databases.

The standard assumption in machine learning is independence and identical distribution (i.i.d.) of data points \cite{Vapnik2000,Cucker2002}. On the other hand, in real operational conditions data are practically never i.i.d. Concept drift, non-identical distributions and various correlations with violation of independence are inevitable. In this paper we try to meet this non-i.i.d. challenge, partially, at least. In particular, the geometric constructions introduced in Sec.~\ref{Sec:DiscrimPrelim} and Theorem~\ref{Prop:ExclVol} do not use the i.i.d. hypothesis. This is an important difference from our previous results, where we assumed i.i.d. data.

In Sec.~\ref{Sec:GeneralSST} we find the general conditions for existence of linear correctors (not compulsory Fisher's discriminants). The essential idea is: the probability that a random data point will belong to a set with small volume (Lebesgue measure) should  also be small (with precise specification, what `small' means). Such a condition is a modified or `smeared' property of absolute continuity (which means that the probability of a set of zero volume is zero). This {\em general stochastic separation theorem} (Theorem~\ref{th:separation}) gives the answer to a closely related  question asked in \cite{DonohoTanner2009}.)

Existence of a linear classifier for corrector construction is a desirable property, and this classifier can be prepared using Support Vector Machine algorithms or Rosenblatt's Perceptron learning rules. Nevertheless, Fisher's linear discriminants seem to be more efficient because they are  non-iterative and robust. In Sections~\ref{Sec:Fisher}-\ref{Sec:QOFish} we analyse Fisher's separability in high dimension. In particular, we prove stochastic separability thorems for log-concave distributions (Sec.~\ref{Sec:log-concave}), and for non-i.i.d. distributions of data points (Sec.~\ref{Sec:QOFish}).

Stochastic separability of real databases is tested in Sec.~\ref{Sec:test}. Popular open access database is used. We calculate, in particular, the probability distribution of $p_y$ that is the probability that a randomly chosen data point $z$ cannot be separated by Fisher's discriminant  from a data point $y$. The value $p_y$ depends on a random data point $y$ and is  a random variable.

The probability $p_y$ that a randomly chosen data point $z$ cannot be separated by Fisher's discriminant  from a data point $y$ depends on a random data point $y$ and is a random variable. Probability distribution of this variable characterises the separability of the data set. We evaluate distribution and moments of $p_y$ and use them for the subsequent analysis of data. The comparison of the mean and variance of $p_y$ with these parameters, calculated for equidistributions in a $n$-dimensional ball or sphere, gives us the idea about what is the real dimension of data. There are many different approaches for evaluation of data dimension. Each definition is needed for specific problems. Here we introduce and use new approach to data dimension   from the Fisher separability analysis point of view.

\subsection{Historical context}

{The {\em curse of dimensionality} is a well-know idea introduced by Bellman in 1957 \cite{Bellman1957}. He considered the problem of multidimensional optimisation and noticed that `the effective analytic of a large number of even simple equations, for example, linear equations, is a difficult affair' and the determination of the maximum `is quite definitely not routine when the number of variables is large'. He used the term `curse'   because it `has hung over the head' for many years and `there is no need to feel discouraged about  the possibility of obtaining significant results despite it.' Many other effects were added to that idea during decades especially in data analysis \cite{Pestov2013}. The idea of `blessing of dimensionality' was expressed much later \cite{Kainen1997,Donoho2000,AndersonEtAl2014,ChenEtAl2013}.}

 {In 2009, Donoho and Tanner described  the blessing of dimensionality effects as {\em three surprises}. The main of them is linear separability of a random point from a large finite random set with high probability \cite{DonohoTanner2009}. They proved this effect for high-dimensional Gaussian i.i.d. samples. In  more general settings,  this effect was supported  by many numerical experiments. This separability was discussed as a very surprising property: `For humans stuck all their lives in three-dimensional space, such a situation is hard to visualize.' These effects have deep connections with the backgrounds of statistical physics and modern geometric functional analysis \cite{GorTyukPhil2018}.}

{In 2016, we added the {\em surprise number four}: this linear separation may be performed by linear Fisher's discriminant \cite{GorTyuRom2016}. This result allows us to decrease significantly  the complexity of the  separation problem: non-iterative (one-shot) approach avoids solution of `a large number of even simple' problems and provides a one more step from Bellman's curse  to the modern blessing of dimensionality. The first theorems were proved in simple settings: i.i.d. samples from the uniform distribution in a multidimensional ball \cite{GorTyuRom2016}. The next step was extension of these theorems onto  i.i.d. samples from the bounded product distributions \cite{GorbTyu2017}. The statements of these two theorems are cited below in Sec.~\ref{Sec:Fisher} (for proofs we refer to \cite{GorbTyu2017}). These results have been supported by various numerical experiments and applications. Nevertheless, the distributions of real life data are far from being i.i.d. samples. They are not distributed uniformly in a ball. The hypothesis of the product distribution is also unrealistic despite of its popularity (in data mining it is known as the `naive Bayes' assumption).}

{We formulated several questions and hypotheses in \cite{GorbTyu2017}.  First of all, we guessed that all {\em essentially high-dimensional} distributions have the same linear (Fisher's) separability properties as the equidistribution in a high-dimensional ball. The question was: how to characterise the class of these   essentially high-dimensional distributions? We also proposed several specific hypotheses about the classes of distributions with linear (Fisher's) separability property. The most important was the hypothesis about {\em log-concave } distributions.  In this paper, we answer the question for characherisation of essentially high dimension distributions both for Fisher's separability  (distributions with bounded support, which satisfy estimate (\ref{bounded}), Theorem~\ref{Theorem:ExclVol2}) and general linear separability (SmAC measures, Definition~\ref{Def:SmAC}, Theorem~\ref{th:separation}, this result answers also to the Donoho-Tanner question). The hypothesis about log-concave-distributions is  proved (Theorem~\ref{th:logconc}). We   try to avoid the i.i.d. hypothesis as far as it was possible (for example, in Theorem~\ref{Theorem:ExclVol2} and special Sec.~\ref{Sec:QOFish}). Some of these results were announced in preprint \cite{GorbanGrechukTyukin2018}.}

{Several important aspects of the problem of stochastic separation in machine learning remain outside the scope of this paper.  First of all, we did not discuss the modern development of the stochastic separability theorems for separation by simple non-linear classifies like small cascades of neurons with independent synaptic weights. Such separation could be orders more effective than the linear separation and still uses non-iterative one-shot algorithms. It was introduced in \cite{GorTyuRom2016}. The problem of separation of sets of data points (not only a point from a set) is  important for applications in machine learning and knowledge transfer between AI systems. The generalised stochastic separation theorems of data sets give the possibility to organise the knowledge transfer without iterations \cite{Tyukin2017a}.}

{An alternative between essentially high-dimensional data with thin shell concentrations,
stochastic separation theorems and efficient linear methods on the one hand, and essentially low-dimensional data with possibly efficient complex nonlinear methods on the other hand, was discussed in \cite{GorTyukPhil2018}.  These two cases could be joined: first, we can extract the most interesting low-dimensional structure and then analyse  the residuals as an essentially high-dimensional random set, which obeys stochastic separation theorems. }

{The trade-off between simple models (and their `interpretability') and more complex non-linear black-box models (and  their `fidelity') was discussed by many authors. Ribeiro at al. proposed to study trade-off between local linear and global nonlinear  classifiers  as a basic example \cite{Ribeiro2016}.  Results on stochastic separation convince that the simple linear discriminants are good global classifiers for high-dimensional data, and complex relations between linear and nonlinear models can reflect the relationships between low-dimensional and high-dimensional components in the variability of the data.  }

{Very recently (after this paper was submitted), a new approach for analysis of classification reliability by the `trust score' is proposed \cite{Jiang2018}. The  `trust score' is closely connected to a heuristic estimate of the dimension of the data cloud by identification of the `high-density-sets', the sets that have high density of points nearby. The samples  with low density nearby are filtered out. In the light of the stochastic separation theorems, this approach can be considered as the extraction  of the essentially low-dimensional  fraction, the points, which are concentrated near a low-dimensional object (a low-dimensional manifold, for example). If this fraction is large and the effective dimension is small, then the low-dimensional  methods can work efficiently. If the low-dimensional fraction is low (or effective dimension is high), then the low-dimensional methods can fail but we enjoy in these situations the blessing of dimensionality with Fisher's discriminants. }

\section{Linear discriminants in high dimension: preliminaries}\label{Sec:DiscrimPrelim}

Throughout the text,  $\Real^n$ is the $n$-dimensional linear real vector space. Unless stated otherwise, symbols $\boldsymbol{x} =(x_{1},\dots,x_{n})$ denote elements of $\Real^n$,  $(\boldsymbol{x},\boldsymbol{y})=\sum_{k} x_{k} y_{k}$ is the inner product of $\boldsymbol{x}$ and $\boldsymbol{y}$, and $\|\boldsymbol{x}\|=\sqrt{(\boldsymbol{x},\boldsymbol{x})}$ is the standard Euclidean norm  in $\Real^n$. Symbol $\mathbb{B}_n$ stands for the unit ball in $\Real^n$ centered at the origin: $\mathbb{B}_n=\{\boldsymbol{x}\in\Real^n| \ \left(\boldsymbol{x},\boldsymbol{x}\right)\leq 1\}$.  $V_n$ is the $n$-dimensional Lebesgue measure, and $V_n(\mathbb{B}_n)$ is the volume of unit ball.  $\mathbb{S}^{n-1}\subset \mathbb{R}^{n}$ is the unit sphere in $\mathbb{R}^{n}$. For a finite set $Y$, the number of points in $Y$ is $|Y|$.

\begin{definition}\label{Def:LinSep}A point $\boldsymbol{x}\in  \mathbb{R}^n$ is {\em linearly separable} from a set $Y \subset \mathbb{R}^n$, if there exists a linear functional $l$ such that $l(\boldsymbol{x})>l(\boldsymbol{y})$ for all $\boldsymbol{y}\in Y$.
\end{definition}

\begin{definition}\label{Def:FishSep}A set $S \subset \mathbb{R}^n$ is {\em linearly separable} or {\em 1-convex} \cite{convhull} if for each $\boldsymbol{x}\in S$ there exists a linear functional $l$ such that $l(\boldsymbol{x})>l(\boldsymbol{y})$ for all $\boldsymbol{y}\in S$, $\boldsymbol{y}\neq \boldsymbol{x}$.
\end{definition}

Recall that a point $\boldsymbol{x} \in K \subset  \mathbb{R}^n$ is an {\em extreme point} of a convex compact $K$ if there exist no points $\boldsymbol{y},\boldsymbol{z}\in K$, $\boldsymbol{y}\neq \boldsymbol{z}$ such that $\boldsymbol{x}=(\boldsymbol{y}+\boldsymbol{z})/2$. The basic examples of linearly separable sets are  extreme points of convex compacts: vertices of convex polyhedra or points  on the $n$-dimensional sphere. Nevertheless, the sets of extreme points of a compact may be not linearly separable  as is demonstrated by simple 2D examples \cite{Simon2011}.

If it is known that a point $\boldsymbol{x}$ is linearly separable from a   finite set $Y$ then the definition of the separating functional $l$ may be performed by linear Support Vector Machine (SVM) algorithms, the Rosenblatt perceptron algorithm or other methods for solving of linear inequalities. These computations may be rather costly and robustness of the result may not be guaranteed.

{With regards to computational complexity, the worst-case estimate for SVM is $O(M^3)$ \cite{bordes2005fast,chapelle2007training}, where $M$ is the number of elements in the dataset. In practice, however, complexity of the soft-margin quadratic support vector machine problem is between $O(M^2)$ and $O(M^3)$, depending on parameters and the problem at hand \cite{bordes2005fast}.  On the other hand, classical Fisher's discriminant requires $O(M)$ elementary operations to construct covariance matrices  followed by $o(n^3)$ operations needed for the $n\times n$ matrix inversion, where $n$ is data dimension. }

 Fisher's linear discriminant is computationally   cheap (after the standard pre-processing),  simple, and robust.

We use a convenient general scheme for creation of Fisher's linear discriminants \cite{Tyukin2017a,GorTyukPhil2018}. For separation of single  points from a data cloud it is necessary:
 \begin{enumerate}
 \item Centralise the cloud (subtract the mean point from all data vectors).
 \item Escape strong multicollinearity, for example, by principal component analysis and deleting minor components, which correspond to the small eigenvalues of empirical covariance matrix.
\item Perform whitening (or spheric transformation), that is a linear transformation, after that the covariance matrix becomes the identity matrix. In principal components, whitening is simply the normalisation of coordinates to unit variance.
\item The linear inequality for separation of a point $\boldsymbol{ x}$ from the cloud $Y $ in new coordinates is
\begin{equation}\label{discriminant}
(\boldsymbol{x},\boldsymbol{y})\leq \alpha (\boldsymbol{x},\boldsymbol{x}),
\end{equation}
 for all $\boldsymbol{y}\in Y$,  where $\alpha\in [0,1)$ is a threshold.
\end{enumerate}

In real-life problems, it could be difficult to perform the precise whitening but a rough approximation to this transformation could also create useful discriminants (\ref{discriminant}). We will call `Fisher's discriminants' all the discriminants created non-iteratively by inner products (\ref{discriminant}), with some extension of  meaning.

\begin{definition}A finite set $F \subset {\mathbb R}^n$ is \emph{Fisher-separable} with threshold $\alpha \in (0,1)$ if inequality (\ref{discriminant})
holds for all $\boldsymbol{ x}, \boldsymbol{ y} \in F$ such that $\boldsymbol{ x}\neq  \boldsymbol{ y}$. The set $F$ is called \emph{Fisher-separable} if there exists some $\alpha\in [0,1)$ such that $F$ is Fisher-separable with threshold $\alpha$. 
\end{definition}

 Fisher's separability implies linear separability but not vice versa.

Inequality (\ref{discriminant}) holds for vectors $\boldsymbol{x}$, $\boldsymbol{y}$ if and only if $\boldsymbol{x}$ does not belong to a ball (Fig.~\ref{Fig:Excluded}) given by the inequality:
\begin{equation}\label{excludedvolume}
\left\{\boldsymbol{z} \ \left| \ \left\|\boldsymbol{z}-\frac{\boldsymbol{y}}{2\alpha }\right\|< \frac{\|\boldsymbol{y}\|}{2\alpha} \right.  \right\}.
\end{equation}
The volume of such balls can be relatively small.

For example, if $Y$ is a subset of $\mathbb{B}_n$, then the volume of each ball (\ref{excludedvolume}) does not exceed $\frac{1}{(2\alpha)^n} V_n(\mathbb{B}_n)$. Point $\boldsymbol{x}$ is separable from a set $Y$ by Fisher's linear discriminant with threshold $\alpha$ if it does not belong to the union of these excluded balls. The volume of this union does not exceed
$$\frac{|Y|}{(2\alpha)^n} V_n(\mathbb{B}_n).$$
Assume that $\alpha> 1/2$.  If $|Y|<b^n$ with $1<b < 2 \alpha$ then the fraction of excluded volume in the unit ball decreases exponentially with dimension $n$ as $\left(\frac{b}{2\alpha}\right)^n$.

\begin{figure}
\centering
\includegraphics[width=0.3\textwidth]{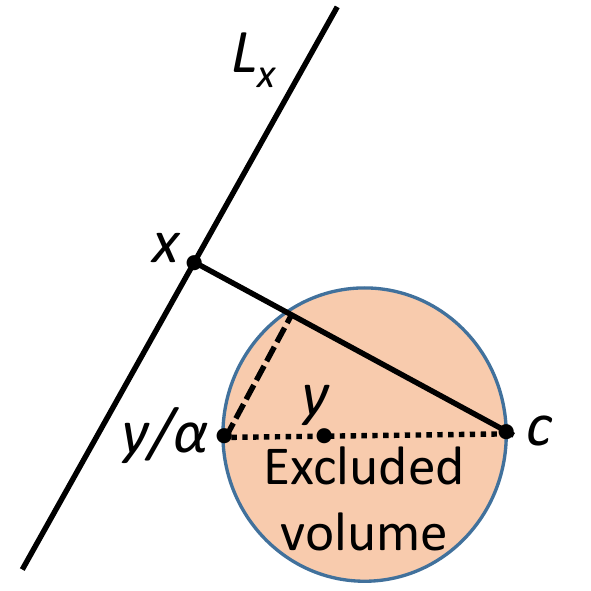}
\caption{Diameter of the filled ball (excluded volume) is the segment $[\boldsymbol{c},\boldsymbol{y}/\alpha]$. Point $\boldsymbol{x}$ should not belong to the excluded volume to be separable from $\boldsymbol{y}$ by the linear discriminant (\ref{discriminant}) with threshold $\alpha$. Here, $\boldsymbol{c}$ is the origin (the centre), and $L_x$ is the hyperplane such that $(\boldsymbol{x},\boldsymbol{z})=(\boldsymbol{x},\boldsymbol{x})$ for $\boldsymbol{z} \in L_x$.  A point $\boldsymbol{x}$ should not belong to the union of such balls for all $\boldsymbol{y} \in Y$ for separability from a set $Y$.}
\label{Fig:Excluded}
\end{figure}

\begin{proposition}\label{Prop:ExclVol}Let $0<\theta<1$, $Y\subset \mathbb{B}_n$ be a finite set, $|Y|<\theta (2\alpha)^n$, and $\boldsymbol{x}$ be a randomly chosen point from the equidistribution in the unit ball. Then with probability $p>1-\theta$ point $\boldsymbol{x}$ is Fisher-separable from $Y$ with threshold $\alpha$ (\ref{discriminant}).
\end{proposition}
This proposition is valid for any finite set  $Y\subset \mathbb{B}_n$ and without any hypotheses about its statistical nature. No underlying probability distribution is assumed. This the main difference from   \cite[Theorem 1]{GorbTyu2017}.

Let us make several remarks for the general distributions. The probability that  a randomly chosen point $\boldsymbol{x}$ is {\em not} separated a given data point from $y$ by the discriminant (\ref{discriminant}) (i.e. that the inequality (\ref{discriminant}) is false) is (Fig.~\ref{Fig:Excluded})
\begin{equation}\label{excluded}
p=p_y=\int_{\left\|\boldsymbol{z}-\frac{\boldsymbol{y}}{2\alpha }\right\|\leq \frac{\|\boldsymbol{y}\|}{2\alpha}} \rho(\boldsymbol{z}) \,d\boldsymbol{z} ,
\end{equation}
where $\rho(\boldsymbol{z}) d\boldsymbol{z}$ is the probability measure.
We need to evaluate the probability of  finding a random point outside the union of $N$ such excluded volumes. For example,  for the equidistribution in a ball $\mathbb{B}_n$, $p_y <\left(\frac{1}{2\alpha}\right)^n$, and arbitrary  $\boldsymbol{y}$ from the ball. The probability to select a point  {inside} the union of $N$ `excluded balls'  is less than $N\left(\frac{1}{2\alpha}\right)^n$ for any allocation of $N$ points $\boldsymbol{y}$ in $\mathbb{B}_n$.

Instead of equidistribution in a ball $\mathbb{B}_n$, we can consider probability distributions  with bounded density $\rho$ in a ball  $\mathbb{B}_n$
\begin{equation}\label{bounded}
\rho(y)<\frac{C}{r^n V_n(\mathbb{B}_n)},
 \end{equation}
where $C>0$ is an arbitrary constant, $V_n(\mathbb{B}_n)$ is the volume of the ball, and   $1>r>1/(2\alpha)$. This inequality guarantees that the probability (\ref{excluded}) for each ball with the radius less or equal than $1/(2\alpha)$ exponentially decays for $n\to \infty$. It should be stressed that the constant $C>0$ is arbitrary but  must not  {depend on $n$} in asymptotic analysis for large $n$.

According to condition (\ref{bounded}), the density $\rho$ is allowed to deviate significantly from the uniform distribution, of which the density is constant, $\rho^*=1/ V_n(\mathbb{B}_n)$. These deviations may increase with $n$ but not faster than the geometric progression with the common ratio $1/r>1$.

The separability properties of distributions, which satisfy the inequality (\ref{bounded}),
are similar to separability for the equidistributions. The proof, through the estimation of the probability to pick a point $\boldsymbol{x}$ in the excluded volume is straightforward.

\begin{theorem}
\label{Theorem:ExclVol2}Let $1\geq \alpha >1/2$, $1>r>1/(2\alpha)$,  $1>\theta>0$, $Y\subset \mathbb{B}_n$ be a finite set, $|Y|<\theta (2r\alpha)^n/C$,  and $\boldsymbol{x}$ be a randomly chosen point from a distribution in the unit ball with the bounded probability density $\rho(\boldsymbol{x})$. Assume that $\rho(\boldsymbol{x})$ satisfies inequality (\ref{bounded}). Then with probability $p>1-\theta$ point $\boldsymbol{x}$ is Fisher-separable from $Y$ with threshold $\alpha$ (\ref{discriminant}).
\end{theorem}
\begin{proof}
The volume of the ball (\ref{excludedvolume}) does not exceed $V=\left(\frac{1}{2\alpha}\right)^nV_n(\mathbb{B}_n)$ for each $\boldsymbol{y}\in Y$. The probability that point  $\boldsymbol{x}$ belongs to such a ball does not exceed
$$V \sup_{z\in \mathbb{B}_n }\rho(z)\leq C \left(\frac{1}{2r \alpha}\right)^n.$$
The probability that $\boldsymbol{x}$   belongs to the union of $|Y|$ such balls does not exceed $|Y| C\left(\frac{1}{2r\alpha}\right)^n$. For $|Y|<\theta (2r\alpha)^n/C$ this probability is smaller than $\theta$ and $p>1-\theta$.

\end{proof}

For many practically relevant cases, the number of points $N$ is large. We have to evaluate the sum $p_y$ for $N$ points $\boldsymbol{y}$.  {For} most estimates, evaluation of the expectations and variances  of  $p_y$ (\ref{excluded}) is sufficient due to the central limit theorem, if points $\boldsymbol{y}\in Y$ are independently chosen. They need not, however, be identically distributed. The only condition of Fisher's separability is that all distributions satisfy   inequality (\ref{bounded}).

For some simple examples, the moments of distribution of $p=p_y$ (\ref{excluded}) can be calculated explicitly. Select $\alpha = 1$ and consider $y$ uniformly distributed in the unit ball $\mathbb{B}_n$. Then, for a given $y$, $p_y=(\|y\|/2)^n$ and for $a<(1/2)^n$:
$$\mathbf{P}(p_y<a)=\mathbf{P}(\|y\|<2a^{1/n}) =2^n a,$$
and $\mathbf{P}(p_y<a)=1$ for $a\geq (1/2)^n$. This is the uniform distribution on the interval $[0,(1/2)^n]$. $\mathbb{E}(p_y)=(1/2)^{n+1}$.

For the equidistribution on the unit sphere $\mathbb{S}^{n-1}$ and   given threshold $\alpha$
\begin{equation}\label{p_ySphere}
p_y=\frac{A_{n-2}}{A_{n-1}}\int_0^{\phi_0} \sin^{n-2}\phi \,d \phi
\end{equation}
(it does not depend on $y$), where $\phi_0=\arccos \alpha$, $A_{m-1}=2\pi^{m/2}/\Gamma[m/2]$ is the area (i.e. the $m-1$-dimensional volume) of the $m-1$-dimensional unit sphere and $\Gamma$ is the gamma-function. Of course, $p_y=0$ for $\alpha =1$ for this distribution because  the sphere is obviously linearly separable. But the threshold matters and $p_y>0$ for $\alpha< 1$.

For large $m$, $$\frac{\Gamma[m/2]}{\Gamma[(m-1)/2]\sqrt{(m-1)/2}}\to 1.$$
Therefore, $$\frac{A_{n-2}\sqrt{\pi(n-2)/2}}{A_{n-1}} \to 1,$$ when $n\to \infty$.

The simple concentration arguments (`ensemble equivalence') allows us to approximate the integral (\ref{p_ySphere}) with the one in which  $\phi\in [\phi_0-\Delta, \phi_0]$. The relative error of this restriction is with exponentially small (in $n$)  for a given $\Delta\ll \phi_0$.
In the vicinity of $\phi_0$
 $$\sin^{n-2}\phi=\exp[(n-2)\ln (\sin \phi)]=\sin^{n-2}\phi_0 \exp[(n-2) (\phi-\phi_0+o(\phi-\phi_0))\cot \phi_0].$$
We apply the same concentration argument again to the simplified integral
$$\int_{\phi_0-\Delta}^{\phi_0}\exp[(n-2)(\phi-\phi_0)\cot \phi_0]  \,d \phi \approx \int_{-\infty}^{\phi_0}\exp[(n-2)(\phi-\phi_0)\cot \phi_0]  \,d \phi = \frac{1}{(n-2) \cot \phi_0 }$$
and derive a closed form estimate: for a given $\phi_0\in (\phi_0,\pi/2)$ and large $n$, $$\int_0^{\psi_0}\sin^{n-2}\phi \, d\phi \approx \frac{ \sin^{n-1}\phi_0}{(n-2)\cos \phi_0}. $$

Therefore, for the equidistribution on the unit sphere  $\mathbb{S}^{n-1}$,
\begin{equation}\label{p_y on sphere}
p_y\approx  \frac{ \sin^{n-1}\phi_0}{\cos \phi_0\sqrt{2\pi (n-2)}}= \frac{(1-\alpha^2)^{(n-1)/2}}{\alpha \sqrt{2\pi (n-2)}}.
\end{equation}
Here $f(n)\approx g(n)$ means (for strictly positive functions) that $f(n)/g(n)\to 1$ when $n\to \infty$.
The probability $p_y$ is the same for all $y \in \mathbb{S}^{n-1}$. It is exponentially small for large $n$.

If the distributions are unknown and exist just in the form of the samples of empirical points, then it is possible to evaluate ${\bf E}(p)$ and ${\rm var}( p)$  (\ref{excluded}) (and the highest moments) from the sample directly, without knowledge of theoretical probabilities. After that, it is possible to apply the central limit theorem and to evaluate the probability that a random point does not belong to the union of the balls (\ref{excludedvolume}) for independently chosen points $y$.

The `efficient dimension' of a data set can be evaluated by comparing $\mathbb{E}(p_y)$ for this data set to  the value of $\mathbb{E}(p_y)$ for the equidistributions on a ball, a sphere, or the Gaussian distribution. Comparison to the sphere is needed if data vectors are normalised to the unit length of each vector (a popular normalisation in some image analysis problems).

Theorem \ref{Prop:ExclVol}, the simple examples, and the previous results \cite{GorbTyu2017} allow us to hypotetise that for essentially high dimensional data it is possible to create correctors of AI systems using simple Fisher's discriminants for separation of areas with high risk of errors from the situations, where the systems work properly. The most important question is: what are the `essentially high dimensional' data distributions? In the following sections we try to answer this question. The leading idea is: the sets of `very' small volume should not have `too high' probability. Specification of these `very' is the main task.

In particular, we find that the hypothesis about stochastic separability for general log-concave distribution proposed in \cite{GorbTyu2017} is true. A general class of probability measures with linear stochastic separability is described (the SmAC distributions). This question was asked in 2009  \cite{DonohoTanner2009}. We demonstrate also how the traditional machine learning assumption about i.i.d. (independent identically distributed) data points can be significantly relaxed.

\section{General stochastic separation theorem}\label{Sec:GeneralSST}

B{\'a}r{\'a}ny and F\"{u}redi \cite{convhull} studied properties of high-dimensional polytopes deriving from uniform distribution in the  {$n$-dimensional} unit ball. They found that in the  envelope of $M$ random  {points} {\em  all } of the points are on the boundary of their convex hull and none belong to the interior (with probability greater than $1-c^2$, provided that $M \leq c 2^{n/2}$, where $c\in (0,1)$ in an arbitrary constant). They also show that the bound on $M$ is nearly tight, up to polynomial factor in $n$.  Donoho and Tanner \cite{DonohoTanner2009} derived a similar result for i.i.d.  points from the Gaussian distribution.
They also mentioned that in applications  it often seems that Gaussianity is not required and stated the problem of characterisation of   ensembles leading to the same qualitative effects (`phase transitions'), which are found  for Gaussian polytopes.

Recently, we noticed that these results could be proven for many other distributions, indeed, and one more important (and surprising) property is also typical: {\em  every point in  this $M$-point random set can be separated from  all other points of this set by the simplest linear Fisher discriminant}  \cite{GorTyuRom2016,GorbTyu2017}. This observation allowed us \cite{GorbanRomBurtTyu2016} to create the corrector technology for legacy AI systems. We  used the `thin shell' measure concentration inequalities to prove these results  \cite{GorbTyu2017,GorTyukPhil2018}. Separation by linear Fisher's discriminant is practically most important {\em Surprise 4} in addition to three surprises mentioned in \cite{DonohoTanner2009}.

The standard approach  {assumes that} the random set consists of independent identically distributed (i.i.d.) random vectors. The new stochastic separation theorem presented below does not assume that the points are identically distributed. It can be very important: in the real practice the new data points are not compulsory taken from the same distribution that the previous points. In that sense the typical situation with the real data flow is far from an i.i.d. sample (we are grateful to G. Hinton for this important remark). This new theorem gives also an answer to the {\em open problem} \cite{DonohoTanner2009}: it gives the general characterisation of the wide class of distributions with stochastic separation theorems (the SmAC condition below). Roughly speaking, this class consists of distributions without sharp peaks in sets with exponentially small volume (the precise formulation is below). We call this property ``SMeared Absolute Continuity'' (or SmAC for short) with respect to the Lebesgue measure: the absolute continuity means that the sets of zero measure have zero probability, and the SmAC condition  below requires that the sets with exponentially small volume should not have high probability.

Consider a  \emph{family} of distributions, one for each pair of positive integers $M$ and $n$. The general SmAC condition is
\begin{definition}\label{Def:SmAC}
The joint distribution of $\boldsymbol{x}_1, \boldsymbol{ x}_2, \dots, \boldsymbol{ x}_M$ has SmAC property if there are exist constants $A>0$, {$B\in(0,1)$}, and $C>0$, such that for every positive integer $n$, any convex set $S \in {\mathbb R}^n$ such that
$$
\frac{V_n(S)}{V_n({\mathbb{B}}_n)} \leq A^n,
$$
any index $i\in\{1,2,\dots,M\}$, and any points $\boldsymbol{ y}_1, \dots, \boldsymbol{ y}_{i-1}, \boldsymbol{ y}_{i+1}, \dots, \boldsymbol{ y}_M$ in ${\mathbb R}^n$,
we have
{\begin{equation}\label{eq:condstar}
{\mathbf{P}}(\boldsymbol{ x}_i \in {\mathbb{B}}_n \setminus S\, | \, \boldsymbol{ x}_j=\boldsymbol{ y}_j, \forall j \neq i) \geq 1-CB^n.
\end{equation}}
\end{definition}

 We remark that
\begin{itemize}
\item We do not require for SmAC condition to hold for \emph{all} $A<1$, just for \emph{some} $A>0$. However, constants $A$, $B$, and $C$ should be independent from $M$ and $n$.
\item  We do not require that $\boldsymbol{ x}_i$ are independent. If they are, \eqref{eq:condstar} simplifies to
$$
{\mathbf{ P}}(\boldsymbol{ x}_i \in {\mathbb{B}}_n \setminus S) \geq 1-CB^n.
$$
\item  We do not require that $\boldsymbol{ x}_i$ are identically distributed.
\item  The unit ball ${\mathbb{B}}_n$ in  SmAC condition can be replaced by an arbitrary ball, due to rescaling.
\item  We do not require the distribution to have a bounded support - points $\boldsymbol{ x}_i$ are allowed to be outside the ball, but with exponentially small probability.
\end{itemize}

The following proposition establishes a sufficient condition for SmAC condition to hold.

\begin{proposition}
Assume that $\boldsymbol{ x}_1, \boldsymbol{ x}_2, \dots, \boldsymbol{ x}_M$ are continuously distributed in ${\mathbb{B}}_n$ with conditional density satisfying
\begin{equation}\label{eq:condden}
\rho_n(\boldsymbol{ x}_i \,|\,\boldsymbol{ x}_j=\boldsymbol{ y}_j, \forall j \neq i) \leq \frac{C}{r^n V_n({\mathbb{B}}_n)}
\end{equation}
for any $n$, any index $i\in\{1,2,\dots,M\}$, and any points $\boldsymbol{ y}_1, \dots, \boldsymbol{ y}_{i-1}, \boldsymbol{ y}_{i+1}, \dots, \boldsymbol{ y}_M$ in ${\mathbb R}^n$, where $C>0$ and $r>0$ are some constants. Then SmAC condition holds with the same $C$, any $B \in (0,1)$, and $A=Br$.
\end{proposition}
\begin{proof}

\begin{equation*}
\begin{split}
{\mathbf{ P}}(\boldsymbol{ x}_i \in S\, |& \, \boldsymbol{ x}_j=\boldsymbol{ y}_j, \forall j \neq i) = \int\limits_S \rho_n(\boldsymbol{ x}_i \,|\,\boldsymbol{ x}_j=\boldsymbol{ y}_j, \forall j \neq i) dV \\& \leq \int\limits_S \frac{C}{r^n V_n({\mathbb{B}}_n)} dV = V_n(S)\frac{C}{r^n V_n({\mathbb{B}}_n)} \\ &\leq A^n V_n({\mathbb{B}}_n)\frac{C}{r^n V_n({\mathbb{B}}_n)} = CB^n.
\end{split}
\end{equation*}

\end{proof}

If $\boldsymbol{ x}_1, \boldsymbol{ x}_2, \dots, \boldsymbol{ x}_M$ are independent with $\boldsymbol{ x}_i$ having density $\rho_{i,n}:{\mathbb{B}}_n \to [0,\infty)$,  \eqref{eq:condden} simplifies to
\begin{equation}\label{eq:indden}
\rho_{i,n}(\boldsymbol{ x}) \leq \frac{C}{r^n V_n({\mathbb{B}}_n)}, \quad \forall n, \, \forall i, \, \forall \boldsymbol{x} \in {\mathbb{B}}_n,
\end{equation}
where $C>0$ and $r>0$ are some constants.

With $r=1$, \eqref{eq:indden} implies that SmAC condition holds for probability distributions whose density is bounded by a constant times density $\rho_n^{uni}:=\frac{1}{V_n({\mathbb{B}}_n)}$ of uniform distribution in the unit ball. With arbitrary $r>0$, \eqref{eq:indden} implies that SmAC condition holds whenever ration $\rho_{i,n}/\rho_n^{uni}$ grows at most exponentially in $n$. This condition is general enough to hold for many distributions of practical interest.

\begin{example}{(Unit ball)}\label{ex:ball}
If $\boldsymbol{ x}_1, \boldsymbol{ x}_2, \dots, \boldsymbol{ x}_M$ are i.i.d random points from the equidistribution in the unit ball, then \eqref{eq:indden} holds with $C=r=1$.
\end{example}
\begin{example}{({Randomly perturbed} data)}\label{ex:noisy}
Fix parameter $\epsilon \in (0,1)$ ({random perturbation} parameter).
Let $\boldsymbol{ y}_1, \boldsymbol{ y}_2, \dots, \boldsymbol{ y}_M$ be the set of $M$ arbitrary (non-random) points inside the ball with radius $1-\epsilon$ in ${\mathbb R}^n$. They might be clustered in arbitrary way, all belong to a subspace of very low dimension, etc. Let
$\boldsymbol{ x}_i, i=1,2,\dots,M$ be a point, selected uniformly at random from a ball with center $\boldsymbol{ y}_i$ and radius $\epsilon$. We think about $\boldsymbol{ x}_i$ as ``perturbed'' version of $\boldsymbol{ y}_i$.
In this model, \eqref{eq:indden} holds with $C=1$, $r=\epsilon$.
\end{example}

\begin{example}{(Uniform distribution in a cube)}\label{ex:cube}
Let $\boldsymbol{ x}_1, \boldsymbol{ x}_2, \dots, \boldsymbol{ x}_M$ be i.i.d random points from the equidistribution in the unit cube. Without loss of generality, we can scale the cube to have side length $s=\sqrt{4/n}$. Then \eqref{eq:indden} holds with $r<\sqrt{\frac{2}{\pi e}}$. 
\end{example}
\begin{remark}
In this case,
\begin{equation*}
\begin{split}
&V_n({\mathbb{B}}_n)\rho_{i,n}(\boldsymbol{ x}) = \frac{V_n({\mathbb{B}}_n)}{(\sqrt{4/n})^n} = \frac{\pi^{n/2}/\Gamma(n/2+1)}{(4/n)^{n/2}}\\& < \frac{(\pi/4)^{n/2}n^{n/2}}{\Gamma(n/2)} \approx \frac{(\pi/4)^{n/2}n^{n/2}}{\sqrt{4\pi/n}(n/2e)^{n/2}} \leq \frac{1}{2\sqrt{\pi}}\left(\sqrt{\frac{\pi e}{2}}\right)^n,
\end{split}
\end{equation*}
where $\approx$ means Stirling's approximation for gamma function $\Gamma$.
\end{remark}

\begin{example}{(Product distribution in unit cube)}\label{ex:product}
Let $\boldsymbol{ x}_1, \boldsymbol{ x}_2, \dots, \boldsymbol{ x}_M$ be independent random points from the product distribution in the unit cube, with component $j$ of point $\boldsymbol{ x}_i$ having a continuous distribution with density $\rho_{i,j}$. Assume that all $\rho_{i,j}$ are bounded from above by some absolute constant $K$. Then \eqref{eq:indden} holds with $r<\frac{1}{K}\sqrt{\frac{2}{\pi e}}$ (after appropriate scaling of the cube).
\end{example}

Below we prove the separation theorem for distributions satisfying SmAC condition. The proof is based on the following result from \cite{volume}.

\begin{proposition}\label{prop:volest}
Let
$$
V(n,M)=\frac{1}{V_n({\mathbb{B}}_n)}\max\limits_{\boldsymbol{ x}_1, \dots, \boldsymbol{x}_M \in {\mathbb{B}}_n} V_n(\text{conv}\{\boldsymbol{x}_1, \dots, \boldsymbol{ x}_M\}),
$$
where $\text{conv}$ denotes the convex hull.
Then
$$
V(n,c^n)^{1/n} < (2e \log c)^{1/2}(1+o(1)), \quad 1<c<1.05.
$$
\end{proposition}
Proposition \ref{prop:volest} implies that for every $c \in (1, 1.05)$, there exists a constant $N(c)$, such that
\begin{equation}\label{eq:volest}
V(n,c^n) < ( {3} \sqrt{\log c})^{n}, \quad n>N(c).
\end{equation}

\begin{theorem}\label{th:separation}
Let $\{\boldsymbol{x}_1, \ldots , \boldsymbol{x}_M\}$ be a set of i.i.d.  random points in ${\mathbb R}^n$ from distribution satisfying SmAC condition. Then $\{\boldsymbol{x}_1, \ldots , \boldsymbol{x}_M\}$ is linearly separable with probability greater than $1-\delta$, $\delta>0$, provided that
$$
M \leq a b^n,
$$
where $$b=\min\{1.05, 1/B, \exp((A/3)^2)\},\;a=\min\{1, \delta/(2C),b^{-N(b)}\}.$$

\end{theorem}
\begin{proof}
If $n<N(b)$, then $M \leq a b^n \leq b^{-N(b)} b^n < 1$, a contradiction. Let $n \geq N(b)$, and let $F=\{\boldsymbol{ x}_1, \boldsymbol{x}_2, \dots, \boldsymbol{ x}_M\}$. Then
\begin{equation*}
{\mathbf{ P}}(F \subset {\mathbb{B}}_n) \geq 1 - \sum_{i=1}^M {\mathbf{ P}}(\boldsymbol{x}_i \not \in {\mathbb{B}}_n)\geq 1 - \sum_{i=1}^M CB^n = 1-MCB^n,
\end{equation*}
where the second inequality follows from \eqref{eq:condstar}.
Next,
\begin{equation*}
{\mathbf{ P}}(F \, \text{is linearly separable}\,|\,F \subset {\mathbb{B}}_n)\geq 1-\sum_{i=1}^M {\mathbf{ P}}(\boldsymbol{ x}_i \in \text{conv}(F\setminus \{\boldsymbol{ x}_i\})\,|\,F \subset {\mathbb{B}}_n).
\end{equation*}
For set $S=\text{conv}(F\setminus \{\boldsymbol{ x}_i\})$
\begin{equation*}
\frac{V_n(S)}{V_n({\mathbb{B}}_n)} \leq  V(n,M-1) \leq V(n,b^n)  < \left(3\sqrt{\log (b)}\right)^n  \leq A^n,
\end{equation*}
where we have used \eqref{eq:volest} and inequalities $a \leq 1$,  $b \leq \exp((A/3)^2)$. Then SmAC condition implies that
\begin{equation*}
{\mathbf{ P}}(\boldsymbol{ x}_i \in \text{conv}(F\setminus \{\boldsymbol{ x}_i\})\,|\,F \subset {\mathbb{B}}_n)
= {\mathbf{ P}}(\boldsymbol{ x}_i \in S\,|\,F \subset {\mathbb{B}}_n) \leq CB^n.
\end{equation*}
Hence,
$$
{\mathbf{ P}}(F \, \text{is linearly separable}\,|\,F \subset {\mathbb{B}}_n) \geq 1 - MCB^n,
$$
and
\begin{equation*}
{\mathbf P}(F \, \text{is linearly separable}) \geq (1 - MCB^n)^2 \geq 1-2MCB^n \geq 1-2ab^{n}CB^n \geq 1-\delta,
\end{equation*}
where the last inequality follows from $a \leq \delta/(2C)$,  $b\leq 1/B$.
\end{proof}

\section{Stochastic separation by Fisher's linear discriminant}\label{Sec:Fisher}

According to the  general stochastic separation theorems there {\em exist}  linear functionals, which  separate points in a random set (with high probability and under some conditions). Such a linear functional can be found by various iterative methods. This possibility is nice but the non-iterative learning is  more beneficial for applications. It would be very desirable to have an explicit expression for separating functionals.

Theorem \ref{Theorem:ExclVol2} and two following theorems  \cite{GorbTyu2017} demonstrate that Fisher's  discriminants are powerful in high dimensions.

\begin{theorem}[Equidistribution in $\mathbb{B}_n$ \cite{GorbTyu2017}]\label{ball1point}
Let $\{\boldsymbol{x}_1, \ldots , \boldsymbol{x}_M\}$ be a set of $M$  i.i.d. random points  from the equidustribution in the unit ball $\mathbb{B}_n$. Let $0<r<1$, and $\rho=\sqrt{1-r^2}$. Then
\begin{equation}
\begin{split}\label{Eq:ball1}
\mathbf{P}&\left(\|\boldsymbol{x}_M\|>r \mbox{ and }
\left(\boldsymbol{x}_i,\frac{\boldsymbol{x}_M}{\| \boldsymbol{x}_M\| } \right)
<r \mbox{ for all } i\neq M \right) \\& \geq 1-r^n-0.5(M-1) \rho^{n};
\end{split}
\end{equation}
\begin{equation}
\begin{split}\label{Eq:ballM}
\mathbf{P}&\left(\|\boldsymbol{x}_j\|>r  \mbox{ and } \left(\boldsymbol{x}_i,\frac{\boldsymbol{x}_j}{\| \boldsymbol{x}_j\|}\right)<r \mbox{ for all } i,j, \, i\neq j\right) \\& \geq  1-Mr^n-0.5M(M-1)\rho^{n};
\end{split}
\end{equation}
\begin{equation}
\begin{split}\label{Eq:ballMangle}
\mathbf{P}&\left(\|\boldsymbol{x}_j\|>r  \mbox{ and } \left(\frac{\boldsymbol{x}_i}{\| \boldsymbol{x}_i\|},\frac{\boldsymbol{x}_j}{\| \boldsymbol{x}_j\|}\right)<r \mbox{ for all } i,j, \,i\neq j\right)\\&  \geq  1-Mr^n-M(M-1)\rho^{n}.
\end{split}
\end{equation}
\end{theorem}

According to Theorem \ref{ball1point}, the probability that a single element $\bfx_M$ from the sample $\mathcal{S}=\{\bfx_1,\dots,\bfx_{M}\}$ is linearly separated from the set $\mathcal{S}\setminus \{\bfx_M\}$ by the hyperplane $l(x)=r$ is at least
\[
1-r^n-0.5(M-1)\left(1-r^2\right)^{\frac{n}{2}}.
\]
This probability estimate depends on both $M=|\mathcal{S}|$ and dimensionality $n$.  An interesting consequence of the theorem is that if one picks a probability value, say $1-\vartheta$, then the maximal possible values of $M$ for which the set $\mathcal{S}$ remains linearly separable with  probability that is no less than $1-\vartheta$ grows at least exponentially with $n$. In particular, the following holds

Inequalities (\ref{Eq:ball1}), (\ref{Eq:ballM}), and (\ref{Eq:ballMangle}) are also closely related to  Proposition~\ref{Prop:ExclVol}.

 \begin{corollary}\label{cor:exponential}
 Let $\{\boldsymbol{x}_1, \ldots , \boldsymbol{x}_M\}$ be a set of $M$ i.i.d. random points  from the equidustribution in the unit ball $\mathbb{B}_n$. Let $0<r,\vartheta<1$, and $\rho=\sqrt{1-r^2}$. If
\begin{equation}\label{EstimateMball}
M<2({\vartheta-r^n})/{\rho^{n}},
 \end{equation}
 then
 $
 \mathbf{P}((\boldsymbol{x}_i,\boldsymbol{x}_M{)}<r\|\boldsymbol{x}_M\| \mbox{ for all } i=1,\ldots, M-1)>1-\vartheta.
$
 If
 \begin{equation}\label{EstimateM2ball}
M<({r}/{\rho})^n\left(-1+\sqrt{1+{2 \vartheta \rho^n}/{r^{2n}}}\right),
 \end{equation}
  then $\mathbf{P}((\boldsymbol{x}_i,\boldsymbol{x}_j)<r\|\boldsymbol{x}_i\| \mbox{ for all } i,j=1,\ldots, M, \, i\neq j)\geq 1-\vartheta.$

  In particular, if inequality (\ref{EstimateM2ball}) holds then the set $\{\boldsymbol{x}_1, \ldots , \boldsymbol{x}_M\}$ is  Fisher-separable  with probability $p>1-\vartheta$.
 \end{corollary}

{Note that (\ref{Eq:ballMangle}) implies that elements of the set $\{\boldsymbol{x}_1, \ldots , \boldsymbol{x}_M\}$ are pair-wise almost or $\varepsilon$-orthogonal, i.e. $|\cos(\bfx_i,\bfx_j)|\leq \varepsilon$ for all $i\neq j$, $1\leq i,j\leq M$,  with probability larger or equal than  $1-2Mr^n-2M(M-1)\rho^{n}$. Similar to Corollary \ref{cor:exponential}, one  can conclude that the cardinality $M$ of samples with such properties grows at least exponentially with $n$. Existence of the phenomenon has been demonstrated in \cite{Kurkova1993}. Theorem \ref{ball1point},  Eq. (\ref{Eq:ballMangle}), shows that the phenomenon is typical in some sense (cf.  \cite{bases}, \cite{Kurkova:2017}).}

{The linear separability property of finite but exponentially large samples of random i.i.d. elements can be proved for various `essentially multidimensional' distributions. It holds for equidistributions in ellipsoids as well as for the Gaussian distributions \cite{GorbanRomBurtTyu2016}. It was generalized in \cite{GorbTyu2017} to product distributions with bounded support. Let the coordinates of the vectors $\bfx=(X_1,\dots,X_n)$ in the set $\mathcal{S}$ be independent random variables $X_i$, $i=1,\dots,n$ with expectations $\overline{X}_i$ and variances $\sigma_i^2>\sigma_0^2>0$. Let $0\leq X_i\leq 1$ for all $i=1,\dots,n$.} 

\begin{theorem}[Product distribution in a cube \cite{GorbTyu2017}]\label{cube} Let  $\{\boldsymbol{x}_1, \ldots , \boldsymbol{x}_M\}$ be i.i.d. random points from the product distribution in a unit cube. Let
\[
R_0^2=\sum_i \sigma_i^2\geq n\sigma_0^2.
\]
Assume that data are centralised and  $0< \delta <2/3$. Then
\begin{equation}\label{Eq:cube1}
\begin{split}
\mathbf{P}&\left(1-\delta  \leq \frac{\|\boldsymbol{x}_j\|^2}{R^2_0}\leq 1+\delta \mbox{ and }
\frac{(\boldsymbol{x}_i,\boldsymbol{x}_M)}{R_0\| \boldsymbol{x}_M\| }<\sqrt{1-\delta}   \mbox{ for all } i,j, \, i\neq M \right) \\ &\geq 1- 2M\exp \left(-2\delta^2 R_0^4/n \right) -(M-1)\exp \left(-2R_0^4(2-3 \delta)^2/n\right);
\end{split}
\end{equation}
\begin{equation}\label{Eq:cube2}
\begin{split}
\mathbf{P}&\left(1-\delta  \leq \frac{\|\boldsymbol{x}_j\|^2}{R^2_0}\leq 1+\delta \mbox{ and }
\frac{(\boldsymbol{x}_i,\boldsymbol{x}_j)}{R_0\| \boldsymbol{x}_j\| }<\sqrt{1-\delta} \mbox{ for all } i,j, \, i\neq j \right) \\& \geq 1- 2M\exp \left(-2\delta^2 R_0^4/n \right) -M(M-1)\exp \left(-2R_0^4(2-3 \delta)^2/n\right).
\end{split}
\end{equation}
\end{theorem}

In particular, under the conditions of Theorem \ref{cube}, set $\{\boldsymbol{x}_1, \ldots , \boldsymbol{x}_M\}$ is Fisher-separable  with probability $p>1-\vartheta$, provided that $M \leq ab^n$, where $a>0$ and $b>1$ are some constants depending only on $\vartheta$ and $\sigma_0$.

{Concentration inequalities in product spaces \cite{Talagrand1995} were employed for the proof  of  Theorem~\ref{cube}.} 

We can see from Theorem \ref{cube} that the discriminant (\ref{discriminant}) works without precise whitening. Just the absence of strong degeneration is required: the support of the distribution contains in the unit cube (that is bounding from above) and, at the same time, the variance of each coordinate is bounded from below by $\sigma_0>0$.

{Linear separability single elements of a set from the rest by Fisher's discriminants is a simple inherent property of high-dimensional data sets. The stochastic separation theorems were generalized further  to account for $m$-tuples, $m>1$ too \cite{Tyukin2017a,GorTyukPhil2018}.}

Bound \eqref{bounded} is the special case of \eqref{eq:indden} with $r>1/2$, and it is  more restrictive: in Examples \ref{ex:noisy}, \ref{ex:cube}, and \ref{ex:product}, the distributions satisfy  \eqref{eq:indden} with some $r<1$, but fail to satisfy \eqref{bounded}  {if $r<1/2$}. Such distributions has the SmAC property and the corresponding set of points $\{\boldsymbol{x}_1, \ldots , \boldsymbol{x}_M\}$ is linearly separable by Theorem \ref{th:separation}, but different technique is needed to establish its Fisher-separability.  {One option is to} estimate the distribution of $p$  {in} (\ref{excluded}).
Another technique is based on concentration inequalities. For some distributions, one can prove that, with exponentially high probability, random point $\boldsymbol{x}$ satisfies
\begin{equation}\label{eq:thinshell}
r_1(n) \leq \|\boldsymbol{x}\| \leq r_2(n),
\end{equation}
where  $r_1(n)$ and $r_2(n)$ are some lower and upper bounds, depending on $n$. If $r_2(n)-r_1(n)$ is small comparing to $r_1(n)$, it means that the distribution is concentrated in a thin shell between the two spheres. If $\boldsymbol{x}$ and $\boldsymbol{y}$ satisfy \eqref{eq:thinshell}, inequality (\ref{discriminant}) may fail only if $\boldsymbol{y}$ belongs to a ball with radius $R=\sqrt{r^2_2(n)-r^2_1(n)}$. If $R$ is much lower than $r_1(n)/2$, this method may provide much better probability estimate than \eqref{excluded}. This is how Theorem \ref{cube} was proved in \cite{GorbTyu2017}.

\section{Separation theorem for log-concave distributions}\label{Sec:log-concave}
\subsection{Log-concave distributions}

Several possible generalisations of Theorems~\ref{ball1point}, \ref{cube} were proposed in \cite{GorbTyu2017}. One of them is the hypothesis that for the uniformly log-concave distributions the similar result can be formulated and proved. Below we demonstrate that this hypothesis is true, formulate and prove the stochastic separation theorems for several classes of log-concave distributions. Additionally, we prove the comparison (domination) Theorem~\ref{prop:domin} that allows to extend the proven theorems to wider classes of distributions.

In this subsection, we introduce several classes of log-concave distributions and prove some useful properties of these distributions.

Let ${\cal P}=\{{\mathbf{ P}}_n, \, n=1,2,\dots\}$ be a family of probability measures with densities $\rho_n:{\mathbb R}^n \to [0,\infty), \, n=1,2,\dots$. Below, $\boldsymbol{x}$ is a random variable (r.v) with density $\rho_n$, and ${\mathbb E}_n[f(\boldsymbol{x})]:=\int_{{\mathbb R}^n} f(z) \rho_n(z) dz$ is the expectation of $f(\boldsymbol{x})$.

We say that density $\rho_n:{\mathbb R}^n \to [0,\infty)$ (and the corresponding probability measure ${\mathbf{ P}}_n$):
\begin{itemize}
\item is whitened, or \emph{isotropic}, if ${\mathbb E}_n[\boldsymbol{x}]=0$, and
\begin{equation}\label{eq:isot}
{\mathbb E}_n[(\boldsymbol{x},\theta)^2)]=1\quad\quad \forall \theta \in \mathbb{S}^{n-1}.
\end{equation}
 The last condition is equivalent to the fact that the covariance matrix of the components of $\boldsymbol{x}$ is the identity matrix  \cite{Lovasz}.
\item is \emph{log-concave}, if set $D_n=\{z\in{\mathbb R}^n \,|\, \rho_n(z)>0\}$ is convex and $g(z)=-\log(\rho_n(z))$ is a convex function on $D_n$.
\item  is \emph{strongly log-concave} (SLC), if $g(z)=-\log(\rho_n(z))$ is strongly convex, that is, there exists a constant $c>0$ such that
$$
\frac{g(u)+g(v)}{2} - g\left(\frac{u+v}{2}\right) \geq c ||u-v||^2, \quad\quad \forall u,v \in D_n.
$$
For example, density $\rho_G(z)=\frac{1}{\sqrt{(2\pi)^n}}\exp\left(-\frac{1}{2}||z||^2\right)$ of $n$-dimensional standard normal distribution is strongly log-concave with $c=\frac{1}{8}$.
\item has sub-Gaussian decay for the norm (SGDN), if there exists a constant $\epsilon>0$ such that
\begin{equation}\label{eq:sgdn}
{\mathbb E}_n\left[\exp\left(\epsilon ||\boldsymbol{x}||^2\right)\right] < +\infty.
\end{equation}
In particular, \eqref{eq:sgdn} holds for $\rho_G$ with any $\epsilon<\frac{1}{2}$.
However, unlike SLC, \eqref{eq:sgdn} is an asymptotic property, and is not affected by local modifications of the underlying density. For example, density $\rho(z)= \frac{1}{C}\exp(-g(||z||)), \, z \in {\mathbb R}^n$, where $g(t)=\frac{1}{2}\max\{1,t^2\}, \, t \in {\mathbb R}$ and $C=\int_{{\mathbb R}^n}\exp(-g(||z||))dz$ has SGDN with any $\epsilon<\frac{1}{2}$, but it is not strongly log-concave.
\item  has sub-Gaussian decay in every direction (SGDD), if there exists a constant $B>0$ such that inequality
\begin{equation}\label{eq:SGDD}
{\mathbf{ P}}_n[(\boldsymbol{x},\theta) \geq t] \leq 2 \exp \left( -\frac{t}{B}\right)^2
\end{equation}
holds for every $\theta \in S^{n-1}$ and $t > 0$.
\item  is $\psi_\alpha$ with constant $B_\alpha > 0$, $\alpha\in[1,2]$, if
\begin{equation}\label{eq:psialpha}
\left({\mathbb E}_n|(\boldsymbol{x},\theta)|^p\right)^{1/p} \leq B_\alpha p^{1/\alpha}\left({\mathbb E}_n|(\boldsymbol{x},\theta)|^2\right)^{1/2}
\end{equation}
holds for every $\theta \in S^{n-1}$ and all $p \geq \alpha$.
\end{itemize}

\begin{proposition}\label{prop:inclusions}
Let $\rho_n:{\mathbb R}^n \to [0,\infty)$ be an isotropic log-concave density, and let $\alpha\in[1,2]$. The following implications hold.
\begin{equation*}
\begin{split}
&\boxed{\rho_n \, \text{is SLC}}
\Rightarrow\boxed{\rho_n \, \text{has SGDN}}
\Rightarrow\boxed{\rho_n \, \text{has SGDD}}\Leftrightarrow \\
&\Leftrightarrow\boxed{\rho_n \, \text{is} \, \psi_2}
\Rightarrow\boxed{\rho_n \, \text{is} \, \psi_\alpha}
\Rightarrow\boxed{\rho_n \, \text{is} \, \psi_1}
\Leftrightarrow\boxed{\text{ALL}}\,\,,
\end{split}
\end{equation*}
where the last $\Leftrightarrow$ means the class of isotropic log-concave densities which are $\psi_1$ actually coincides with the class of \emph{all} isotropic log-concave densities.
\end{proposition}
\begin{proof}
Proposition 3.1 from \cite{Bobkov} states that if there exists $c_1>0$ such that
$g(\boldsymbol{x})=-\log(\rho_n(\boldsymbol{x}))$ satisfies
\begin{equation}\label{eq:allts}
t g(u)+s g(v) - g\left(t u+s v\right) \geq \frac{c_1ts}{2} ||u-v||^2, \;  \forall u,v \in D_n.
\end{equation}
for all $t,s>0$ such that $t+s=1$, then inequality
\begin{equation}\label{eq:sobolev}
{\mathbb E}_n[f^2(\boldsymbol{x})\log f^2(\boldsymbol{x})] - {\mathbb E}_n[f^2(\boldsymbol{x})]{\mathbb E}_n[\log f^2(\boldsymbol{x})]  \leq \frac{2}{c_1} {\mathbb E}_n[||\nabla f(\boldsymbol{x})||^2]
\end{equation}
holds for every smooth function $f$ on ${\mathbb R}^n$. As remarked in \cite[p. 1035]{Bobkov}, ``it is actually enough that \eqref{eq:allts} holds for some $t, s > 0, t+s = 1$''. With $t=s=1/2$, this implies that \eqref{eq:sobolev} holds for every strongly log-concave distribution, with $c_1=8c$. According to \cite[Theorem 3.1]{Bobkov2}, \eqref{eq:sobolev} holds for $\rho_n$ if and only if it has has sub-Gaussian decay for the norm, and the implication $\boxed{\rho_n \, \text{is SLC}}\Rightarrow\boxed{\rho_n \, \text{has SGDN}}$ follows.

According to \citet[Theorem 1(i)]{Stavrakakis}, if \eqref{eq:sobolev} holds for $\rho_n$, then it is $\psi_2$ with constant $B_2=d/\sqrt{c_1}$, where $d$ is a universal constant, hence $\boxed{\rho_n \, \text{has SGDN}}\Rightarrow \boxed{\rho_n \, \text{is} \, \psi_2}\,$.

Lemma 2.4.4 from \cite{Brazitikos} implies that if log-concave $\rho_n$ is $\psi_\alpha$ with constant $B_\alpha$ then inequality
\begin{equation}\label{eq:SGDDgen}
{\mathbf{ P}}_n[(\boldsymbol{x},\theta) \geq t ({\mathbb E}_n[(\boldsymbol{x},\theta)^2)])^{1/2}] \leq 2 \exp \left( -\frac{t}{B}\right)^\alpha
\end{equation}
holds for all $\theta \in S^{n-1}$ and all $t>0$, with constant $B = B_\alpha$. Conversely, if \eqref{eq:SGDDgen} holds for all $\theta \in S^{n-1}$ and all $t>0$, then $\rho_n$ is $\psi_\alpha$ with constant $B_\alpha=C B$, where $C$ is a universal constant.
For isotropic $\rho_n$, \eqref{eq:isot} implies that \eqref{eq:SGDDgen} with $\alpha=2$ simplifies to \eqref{eq:SGDD}, and the equivalence $\boxed{\rho_n \, \text{has SGDD}}\Leftrightarrow\boxed{\rho_n \, \text{is} \, \psi_2}$ follows.

The implications $\boxed{\rho_n \, \text{is} \, \psi_2} \Rightarrow\boxed{\rho_n \, \text{is} \, \psi_\alpha} \Rightarrow\boxed{\rho_n \, \text{is} \, \psi_1}$ follow from \eqref{eq:SGDDgen}.

Finally, according to \cite[Theorem 2.4.6]{Brazitikos}, inequalities
$$
\left({\mathbb E}_n|f|^q\right)^{1/q} \leq \left({\mathbb E}_n|f|^p\right)^{1/p} \leq B_1\frac{p}{q}\left({\mathbb E}_n|f|^q\right)^{1/q}
$$
hold for any seminorm $f:{\mathbb R}^n \to {\mathbb R}$, and any $p>q \geq 1$, where $B_1$ is a universal constant. Because $f(\boldsymbol{x})=|(\boldsymbol{x},\theta)|$ is a seminorm for every $\theta \in S^{n-1}$, this implies that
$$
\left({\mathbb E}_n|(\boldsymbol{x},\theta)|^p\right)^{1/p} \leq B_1 p {\mathbb E}_n|(\boldsymbol{x},\theta)| \leq  B_1 p\left({\mathbb E}_n|(\boldsymbol{x},\theta)|^2\right)^{1/2}
$$
which means that every log-concave density $\rho_n$ is $\psi_1$ with some universal constant.
 \end{proof}

\subsection{Fisher-separability for log-concave distributions}

Below we prove Fisher-separability for i.i.d samples from isotropic log-concave $\psi_\alpha$ distributions, using the  technique based on concentration inequalities.
\begin{theorem}\label{th:logconc}
Let $\alpha \in [1,2]$, and let
${\cal P}=\{{\mathbf{ P}}_n, \, n=1,2,\dots\}$ be a family of isotropic log-concave probability measures with densities $\rho_n:{\mathbb R}^n \to [0,\infty), \, n=1,2,\dots$,
which are $\psi_\alpha$ with constant $B_\alpha > 0$, independent from $n$.
Let $\{\boldsymbol{x}_1, \ldots , \boldsymbol{x}_M\}$ be a set of $M$ i.i.d. random points from $\rho_n$. Then there exist constants $a>0$ and $b>0$, which depends only on $\alpha$ and $B_\alpha$, such that, for any $i,j \in \{1,2,\dots,M\}$, inequality
$$
(\boldsymbol{ x_i},\boldsymbol{ x_i}) > (\boldsymbol{ x_i},\boldsymbol{ x_j})
$$
holds with probability at least $1-a\exp(-b n^{\alpha/2})$. Hence, for any $\delta>0$, set $\{\boldsymbol{x}_1, \ldots , \boldsymbol{x}_M\}$ is Fisher-separable with probability greater than $1-\delta$, provided that
\begin{equation}\label{eq:Mbound}
M \leq \sqrt{\frac{2\delta}{a}}\exp\left(\frac{b}{2}n^{\alpha/2}\right).
\end{equation}
\end{theorem}
\begin{proof}
Let $\boldsymbol{x}$ and $\boldsymbol{y}$ be two points, selected independently at random from the distribution with density $\rho_n$.

\cite[Theorem 1.1]{Guedon}, (applied with $A=I_n$, where $I_n$ is $n \times n$ identity matrix), states that for any $t\in(0,1)$,
\eqref{eq:thinshell} holds with $r_1(n)=(1-t)\sqrt{n}$, $r_2(n)=(1+t)\sqrt{n}$, and with probability at least $1-A\exp(-Bt^{2+\alpha}n^{\alpha/2})$, where $A,B>0$ are constants depending only on $\alpha$. If \eqref{eq:thinshell} holds for $\boldsymbol{x}$ and $\boldsymbol{y}$, inequality $(\boldsymbol{x},\boldsymbol{y})\leq(\boldsymbol{x},\boldsymbol{x})$ may fail only if $\boldsymbol{y}$ belongs to a ball with radius $R_n=\sqrt{r^2_2(n)-r^2_1(n)}=\sqrt{4tn}$.

\cite[Theorem 6.2]{Paouris}, applied with $A=I_n$,
states that, for any $\epsilon\in(0,\epsilon_0)$, $\boldsymbol{y}$ does \emph{not} belong to a ball with any center and radius $\epsilon\sqrt{n}$, with probability at least $1-\epsilon^{Cn^{\alpha/2}}$ for some constants $\epsilon_0>0$ and $C>0$. By selecting $t=\epsilon_0^2/8$, and $\epsilon=\sqrt{4t}=\epsilon_0/2$, we conclude that (\ref{discriminant}) holds with probability at least $1-2A\exp(-Bt^{2+\alpha}n^{\alpha/2})-(\sqrt{4t})^{Cn^{\alpha/2}}$. This is greater than $1-a\exp(-b n^{\alpha/2})$ for some constants $a>0$ and $b>0$. Hence, $\boldsymbol{ x}_1, \boldsymbol{ x}_2, \dots,\boldsymbol{x}_M$ are Fisher-separable with probability greater than $1-\frac{M(M-1)}{2}a\exp(-b n^{\alpha/2})$. This is greater than $1-\delta$ provided that $M$ satisfies \eqref{eq:Mbound}.
\end{proof}

\begin{corollary}\label{cor:sqrtn}
Let $\{\boldsymbol{x}_1, \ldots , \boldsymbol{x}_M\}$ be a set of $M$ i.i.d. random points from an isotropic log-concave distribution in ${\mathbb R}^n$. Then set $\{\boldsymbol{x}_1, \ldots , \boldsymbol{x}_M\}$ is Fisher-separable with probability greater than $1-\delta$, $\delta>0$, provided that
$$
M \leq a c^{\sqrt{n}},
$$
where $a>0$ and $c>1$ are constants, depending only on $\delta$.
\end{corollary}
\begin{proof}
This follows from Theorem \ref{th:logconc} with $\alpha=1$ and the fact that all log-concave densities are $\psi_1$ with some universal constant, see Proposition \ref{prop:inclusions}.
\end{proof}

We say that family ${\cal P}=\{{\mathbf{ P}}_n, \, n=1,2,\dots\}$ of probability measure has \emph{exponential Fisher separability} if there exist constants $a>0$ and $b\in(0,1)$ such that, for all $n$, inequality \eqref{discriminant} holds with probability at least $1-ab^n$, where $\boldsymbol{ x}$ and $\boldsymbol{ y}$ are i.i.d vectors in ${\mathbb R}^n$ selected with respect to ${\mathbf{ P}}_n$. In this case, for any $\delta>0$, $M$ i.i.d vectors $\{\boldsymbol{x}_1, \ldots , \boldsymbol{x}_M\}$ are Fisher-separable with probability at least $1-\delta$ provided that
$$
M \leq \sqrt{\frac{2\delta}{a}}\left(\frac{1}{\sqrt{b}}\right)^n.
$$

\begin{corollary}\label{cor:psi2}
Let ${\cal P}=\{{\mathbf{ P}}_n, \, n=1,2,\dots\}$ be a family of isotropic log-concave probability measures which are all $\psi_2$ with the same constant $B_2>0$.
Then ${\cal P}$ has exponential Fisher separability.
\end{corollary}
\begin{proof}
This follows from Theorem \ref{th:logconc} with $\alpha=2$ .
\end{proof}

\begin{corollary}\label{cor:strconc}
Let ${\cal P}=\{{\mathbf{ P}}_n, \, n=1,2,\dots\}$ be a family of isotropic probability measures
which are all strongly log-concave with the same constant $c>0$. Then ${\cal P}$ has exponential Fisher separability.
\end{corollary}
\begin{proof}
The proof of Proposition \ref{prop:inclusions} implies that ${\mathbf{ P}}_n$ are all $\psi_2$ with the same constant  $B_2=d/\sqrt{c}$, where $d$ is a universal constant. The statement then follows from Corollary \ref{cor:psi2}.
\end{proof}

\begin{example}\label{cor:standnorm}
Because standard normal distribution in $\mathbb R_n$ is strongly log-concave with $c=\frac{1}{8}$, Corollary \ref{cor:strconc} implies that the family of standard normal distributions has exponential Fisher separability.
\end{example}

\subsection{Domination}

We say that family ${\cal P}'=\{{\mathbf{ P}}'_n, \, n=1,2,\dots\}$ dominates family ${\cal P}=\{{\mathbf{ P}}_n, \, n=1,2,\dots\}$ if there exists a constant $C$ such that
\begin{equation}\label{eq:domin}
{\mathbf{ P}}_n (S) \leq C \cdot {\mathbf{ P}}'_n (S)
\end{equation}
holds for all $n$ and all measurable subsets $S \subset {\mathbb R}^n$. In particular, if ${\mathbf{ P}}'_n$ and ${\mathbf{ P}}_n$ have densities $\rho'_n:{\mathbb R}^n \to [0,\infty)$ and $\rho_n:{\mathbb R}^n \to [0,\infty)$, respectively, then \eqref{eq:domin} is equivalent to
\begin{equation}\label{eq:domindens}
\rho_n (\boldsymbol{ x}) \leq C \cdot \rho'_n (\boldsymbol{ x}), \quad \forall \boldsymbol{ x} \in {\mathbb R}^n.
\end{equation}

\begin{theorem}\label{prop:domin}
If family ${\cal P}'$ has exponential Fisher separability, and ${\cal P}'$ dominates ${\cal P}$, then ${\cal P}$ has exponential Fisher separability.
\end{theorem}
\begin{proof}
For every $\boldsymbol{ x}=(x_1, \dots, x_n) \in {\mathbb R}^{n}$ and $\boldsymbol{ y}= (y_1, \dots, y_n)\in {\mathbb R}^{n}$, let $\boldsymbol{ x} \times \boldsymbol{ y}$ be a point in ${\mathbb R}^{2n}$ with coordinates $(x_1, \dots, x_n, y_1, \dots, y_n)$.
Let ${\mathbf{ Q}}_n$ be the product measure of ${\mathbf{ P}}_n$ with itself, that is, for every measurable set $S \subset {\mathbb R}^{2n}$, ${\mathbf{ Q}}_n (S)$ denotes the probability that $\boldsymbol{ x} \times \boldsymbol{ y}$ belongs to $S$, where vectors $\boldsymbol{ x}$ and $\boldsymbol{ y}$ are i.i.d vectors selected with respect to ${\mathbf{ P}}_n$.
Similarly, let ${\mathbf{ Q}}'_n$ be the product measure of ${\mathbf{ P}}'_n$ with itself. Inequality \eqref{eq:domin} implies that
$$
{\mathbf{ Q}}_n (S) \leq C^2 \cdot {\mathbf{ Q}}'_n (S), \quad \forall S \subset {\mathbb R}^{2n}.
$$
Let $A_n \subset {\mathbb R}^{2n}$ be the set of all $\boldsymbol{ x} \times \boldsymbol{ y}$ such that $(\boldsymbol{ x},\boldsymbol{ x}) \leq (\boldsymbol{ x},\boldsymbol{ y})$. Because ${\cal P}'$ has exponential Fisher separability, ${\mathbf{ Q}}'_n (A_n) \leq ab^n$ for some $a>0$, $b\in(0,1)$. Hence,
$$
{\mathbf{ Q}}_n (A_n) \leq C^2 \cdot {\mathbf{ Q}}'_n (A_n) \leq (aC^2)b^n,
$$
and exponential Fisher separability of ${\cal P}$ follows.
\end{proof}

\begin{corollary}\label{cor:domnorm}
Let ${\cal P}=\{{\mathbf{ P}}_n, \, n=1,2,\dots\}$ be a family of distributions which is dominated by a family of (possibly scaled) standard normal distributions. Then ${\cal P}$ has exponential Fisher separability.
\end{corollary}
\begin{proof}
This follows from Example \ref{cor:standnorm}, Theorem \ref{prop:domin}, and the fact that scaling does not change Fisher separability.
\end{proof}

\section{Quasiorthogonal sets and Fisher separability of not i.i.d. data}\label{Sec:QOFish}

The technique based on concentration inequalities usually fails if the data are not identically distributed, because, in this case, each $\boldsymbol{x}_i$ may be concentrated in its \emph{own} spherical shell. An alternative approach to prove separation theorems is to use the fact that, in high dimension, almost all vectors are almost orthogonal \cite{bases}, which implies that $(\boldsymbol{x},\boldsymbol{y})$ in (\ref{discriminant}) is typically ``small''. Below we apply this idea to prove Fisher separability of exponentially large families in the {``randomly perturbed''} model described in Example \ref{ex:noisy}.

Consider the {``randomly perturbed''} model from  Example \ref{ex:noisy}. In this model,  Fisher's hyperplane for separation each point $\boldsymbol{x}_i$ will be calculated assuming that coordinate center is the corresponding cluster centre $\boldsymbol{y}_i$.
\begin{theorem}\label{th:noisy}
Let $\{\boldsymbol{x}_1, \ldots , \boldsymbol{x}_M\}$ be a set of $M$ random points in the {``randomly perturbed''} model (see Example \ref{ex:noisy}) with {random perturbation} parameter $\epsilon>0$.
For any $\frac{1}{\sqrt{n}} < \delta < 1$, set $\{\boldsymbol{x}_1, \ldots , \boldsymbol{x}_M\}$ is Fisher-separable with probability at least
$$
1 - \frac{2M^2}{\delta\sqrt{n}}\left(\sqrt{1-\delta^2}\right)^{n+1}-M\left(\frac{2\delta}{\epsilon}\right)^n.
$$
In particular, set $\{\boldsymbol{x}_1, \ldots , \boldsymbol{x}_M\}$ is Fisher-separable with probability at least $1-v$, $v>0$, provided that $M<a b^n$, where $a,b$ are constants depending only on $v$ and $\epsilon$.
\end{theorem}
\begin{proof}
Let $\boldsymbol{x} \in {\mathbb R}^n$ be an arbitrary non-zero vector, and let $\boldsymbol{u}$ be a vector selected uniformly at random from a unit ball. Then, for any $\frac{1}{\sqrt{n}}<\delta<1$,
\begin{equation}\label{eq:almostort}
P\left(\left|\left(\frac{\boldsymbol{x}}{||\boldsymbol{x}||}, \frac{\boldsymbol{u}}{||\boldsymbol{u}||}\right)\right| \geq \delta \right) \leq \frac{2}{\delta\sqrt{n}}\left(\sqrt{1-\delta^2}\right)^{n+1},
\end{equation}
see  \cite[Lemma 4.1]{Lovasz}.

Applying \eqref{eq:almostort} to $\boldsymbol{u}=\boldsymbol{x}_i-\boldsymbol{y}_i$, we get
$$
P\left(\left|\left(\frac{\boldsymbol{x}_j}{||\boldsymbol{x}_j||}, \frac{\boldsymbol{u}}{||\boldsymbol{u}||}\right)\right| \geq \delta \right) \leq \frac{2}{\delta\sqrt{n}}\left(\sqrt{1-\delta^2}\right)^{n+1}, \;  j\neq i,
$$
and also
$$
P\left(\left|\left(\frac{\boldsymbol{y}_i}{||\boldsymbol{y}_i||}, \frac{\boldsymbol{u}}{||\boldsymbol{u}||}\right)\right| \geq \delta \right) \leq \frac{2}{\delta\sqrt{n}}\left(\sqrt{1-\delta^2}\right)^{n+1}.
$$
On the other hand
$$
P\left(||\boldsymbol{x}_i-\boldsymbol{y}_i|| \leq 2\delta\right) = \left(\frac{2\delta}{\epsilon}\right)^n.
$$

If none of the listed events happen, then projections of all points $\boldsymbol{x}_j$, $j \neq i$, on $\boldsymbol{u}$ have length at most $\delta$ (because $||\boldsymbol{x}_j|| \leq 1, \forall j$), while the length of projection of $\boldsymbol{x}_i$ on $\boldsymbol{u}$ is greater than $\delta$, hence $\boldsymbol{x}_i$ is separable from other points by Fisher discriminant (with center $\boldsymbol{y}_i$). Hence, the probability that $\boldsymbol{x}_i$ is not separable is at most
$$
\frac{2M}{\delta\sqrt{n}}\left(\sqrt{1-\delta^2}\right)^{n+1}+\left(\frac{2\delta}{\epsilon}\right)^n
$$
The probability that there exist some index $i$ such that $\boldsymbol{x}_i$ is not separable is at most the same expression multiplied by $M$.
\end{proof}

{Theorem \ref{th:noisy} is yet another illustration of why randomization and randomized approaches to learning may improve performance of AI systems (see e.g. works by \citet{Wang2017} and \citet{Wang2017a} for more detailed discussion on the randomized approaches and supervisory mechanisms for random parameter assignment).}

{Moreover,} Theorem \ref{th:noisy} shows that  the cluster structure of data is not an insurmountable obstacle for separation theorems. The practical experience ensures us that combination of cluster analysis with stochastic separation theorems works much better than the stochastic separation theorems directly, if there exists a pronounced cluster structure in data. The preferable way of action is:
\begin{itemize}
\item Find clusters in data clouds;
\item Create classifiers for distribution of newly coming data between clusters;
\item Apply stochastic separation theorems with discriminant (\ref{discriminant}) for each cluster separately.
\end{itemize}

This is a particular case of the general rule about complementarity between low-dimensional non-linear structures and high-dimensional stochastic separation \cite{GorTyukPhil2018}.

\section{Stochastic separability in a real database: A LFW case study}\label{Sec:test}

\paragraph*{Data} LFW (Labeled Faces in the Wild) is a data set of faces of famous people like politicians, actors and singers \cite{LFWSurvey2016}. LFW includes 13,233 photos of 5,749 people. There are 10,258 males and 2,975 females in the database. The data are available online \cite{LFW}.

\paragraph*{Preprocessing} Every photo from the data set was {\em cropped and aligned}. The 128-dimensional {\em feature vector was calculated} by FaceNet for the cropped and aligned images. FaceNet learns a mapping from face images to a 128-dimensional Euclidean space where distances reflect face similarity \cite{FaceNet2015}. After that, we worked with 128-dimensional data vectors instead of initial images.

The next essential dimension decrease was provided by the standard {\em Principal Component analysis} (PCA) preceded by {\em centering} (to zero mean) and {\em normalization} (to unit standard deviation). Figure~\ref{variance} presents dependence of the eigenvalue $\lambda_k$ on $k$ (Fig.~\ref{variance} a))   and the cumulative plot of variance explained by $k$ first principal components as a function of $k$ (Fig.~\ref{variance} b)). There exist various heuristics for selection of main components \cite{Cangelosi2007}. It is obvious from the plot that for any reasonable estimate not more than 63 principal components are needed, which explain practically 100\% of data variance. We used the {\em multicollinearity control}. Multicollinearity means strong linear dependence between input attributes. This property makes the model very sensitive to fluctuations in data and is an important source of instability. The standard measure of multicollinearity  is the condition number of the correlation matrix, that is the ratio $\kappa=\lambda_{\max}/\lambda_{\min}$, where $\lambda_{\max}$ and $\lambda_{\min}$ are the maximal and the minimal eigenvalues of this matrix. Collinearity with $\kappa<10$ is considered as `modest' collinearity  \cite{Dormann2013}. Therefore, we selected the principal components with the eigenvalues of covariance matrix $\lambda\leq 0.1\lambda_{\max}$. There are 57 such components.

\begin{figure}
\centering
a)\includegraphics[width=0.45\textwidth]{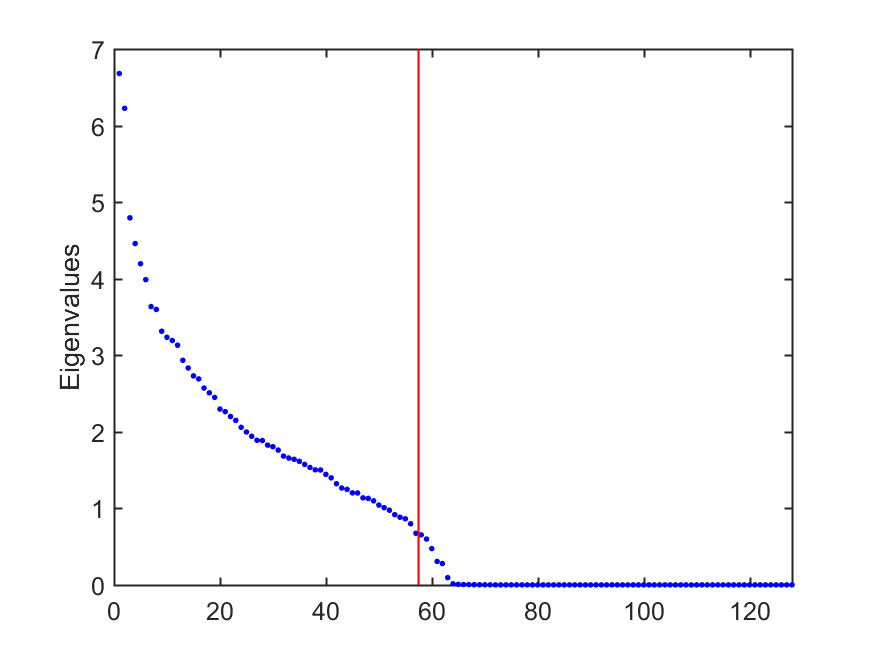} \;\; b)\includegraphics[width=0.45\textwidth]{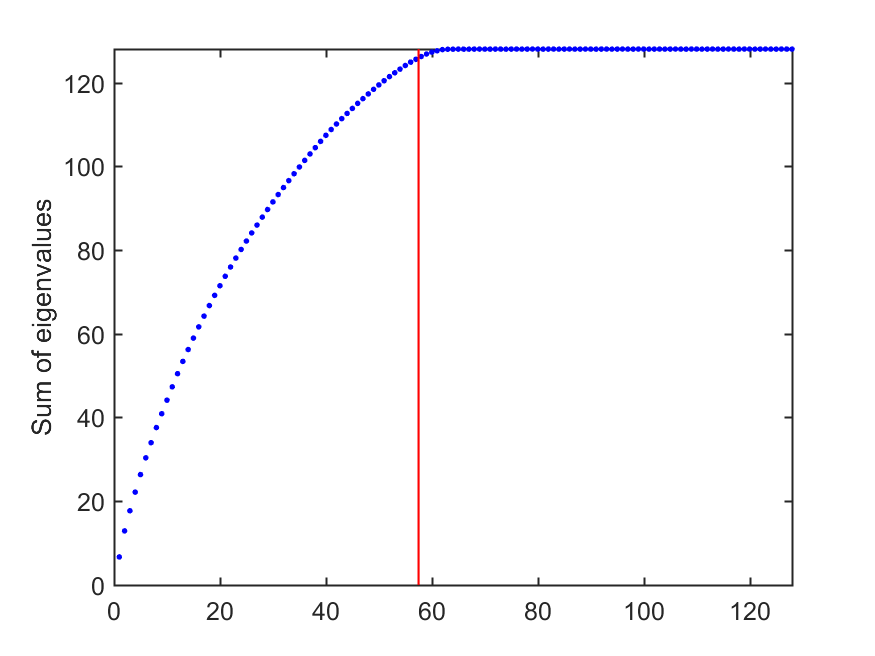}
\caption{Eigenvalue of the correlation matrix as a function of the number of principal component (a). Cumulative plot of the sum of the eigenvalues (variance explained) as a function of their number (b). The verticaL line separates first 57 eigenvalues.}
\label{variance}
\end{figure}

After projection on these 57 principal components,  {\em whitening} was applied in the 57-dimensional space. After whitening, the covariance matrix became identity matrix. In the principal component basis, whitening is just a coordinate transformation: $x_i \to x_i/\sqrt{\lambda_i}$, where $x_i $ is the projection of the data vector on $i$th principal component, and $\lambda_i$ is the eigenvalue  of the covariance matrix, correspondent to the $i$th principal component.

Optionally, additional preprocessing operation was applied after whitening, the normalization of each data vector to unit length (just take $x/\|x\|$ instead of $x$). This operation can be called the projection on the unit sphere.

\paragraph*{Separability analysis} The results of computation are presented in Table~\ref{tab:separability}. Here,  $N_{\alpha}$ is the number of points, which cannot be separated from the rest of the data set by Fisher's discriminant (\ref{discriminant}) with given $\alpha$, $\nu_{\alpha}$ is the fraction of these points in the data set, and $\bar{p_y}$ is the mean value of $p_y$.

\begin{table}[!ht]
\centering
\caption{Separability of data points by Fisher's linear discriminant in the preprocessed LFW data set.}
\label{tab:separability}
\begin{tabular}{|l|c|c|c|c|c|}
\cline{1-6}
$\alpha$ &	0.8&0.9&0.95&0.98&0.99  \\ \hline
\multicolumn{6}{|c|}{Separability from all data} \\ \hline
$N_{\alpha}$   &4058&751&123&26&10 \\\hline
$\nu_{\alpha}$  &0.3067&0.0568&0.0093&0.0020&0.0008\\\hline
$\bar{p_y}$  &9.13E-05&7.61E-06&8.91E-07&1.48E-07&5.71E-08 \\\hline
	\multicolumn{6}{|c|}{Separability from points of different classes}\\\hline
$N_{\alpha}^*$   &55&13&6&5&5 \\\hline
$\nu_{\alpha}^*$  &0.0042&0.0010&0.0005&0.0004&0.0004\\\hline
$\bar{p_y}^*$  &3.71E-07&7.42E-08&3.43E-08&2.86E-08&2.86E-08 \\\hline
\multicolumn{6}{|c|}{Separability from all data on unit sphere} \\ \hline
$N_{\alpha}$   &3826&475&64&12&4 \\\hline
$\nu_{\alpha}$  &0.2891&0.0359&0.0048&0.0009&0.0003\\\hline
$\bar{p_y}$  &7.58E-05&4.08E-06&3.66E-07&6.85E-08&2.28E-08 \\\hline
	\multicolumn{6}{|c|}{Separability from points of different classes on unit sphere}\\\hline
$N_{\alpha}^*$   &37&12&8&4&4 \\\hline
$\nu_{\alpha}^*$  &0.0028&0.0009&0.0006&0.0003&0.0003\\\hline
$\bar{p_y}^*$  &2.28E-07&6.85E-08&4.57E-08&2.28E-08&2.28E-08 \\\hline
\end{tabular}
\end{table}

We cannot expect an i.i.d. data sampling for a `good' distribution to be an appropriate model for the LFW data set. Images of the same person are  expected to have more similarity between them than images of different persons. It is expected that this set of data in FaceNet coordinates will be even further from the reasonable representation in the form of i.i.d. sampling of data, because it is prepared to group images of the same person together and `repel' images of different persons. Consider the property of image to be `separable from images of other persons'. Statistics of this `separability from all points of different classes' by Fisher's discriminant is also presented in Table~\ref{tab:separability}. We call  the point $\boldsymbol{x}$ inseparable from points of different classes if at least for one point
$\boldsymbol{y}$ of a different class (image of a different person)
$(\boldsymbol{x},\boldsymbol{y})> \alpha (\boldsymbol{x},\boldsymbol{x})$.
We use stars in Table for statistical data about separability  from points of different classes (i.e., use the notations $N_{\alpha}^*$, $\nu_{\alpha}^*$, and $\bar{p_y}^*$).

It is not surprising that this separability is much more efficient, with less inseparable points. Projection on the unit sphere also improves separability.

It is useful to introduce some baselines for comparison: what values of  $\bar{p_y}$ should be considered as small or large? Two levels are obvious: for the experiments, where all the data points are counted in the excluded volume, we consider the level $p=1/N_{\rm persons}$ as the critical one, where $N_{\rm persons}$ is the number  of different persons in the database. For the experiments, where only images of different persons are counted in the excluded volume, the value $1/N$ seems to be a good candidate to separate the `large' values of  $\bar{p_y}^*$ from the `small' ones. For LFW  data set, $1/N_{\rm persons}=$1.739E-04 and $1/N=$7.557E-05. Both levels have achieved already for $\alpha=0.8$. The parameter $\nu_{\alpha}^*$ can be considered for experiments with separability from points of different classes as an error of separability. This error is impressively small: already for $\alpha=0.8$ it is less than 0.5\% without projection on unit sphere and less than 0.3\% with this projection. The ratio $(N_{\alpha}-N_{\alpha}^*)/N_{\alpha}^*=(\nu_{\alpha}-\nu_{\alpha}^*)/\nu_{\alpha}^*$ can be used to evaluate   the generalisation ability of the Fisher's linear classifier. The nominator is the number of data points (images) inseparable from some points of the same class (the person) but separated from the points of all other classes. For these images we can say that Fisher's discriminant makes some generalizations. The denominator is the number of data points inseparable from some points of other classes (persons). According to the proposed indicator, the generalization ability is impressively high for $\alpha=0.8$ (it is 72.8 for preprocessing without projection onto unit sphere and 102.4 for prerpocessing with such projection. For $\alpha=0.9$ it is smaller (56.8 and 38.6 correspondingly) and the decays fast with growing of $\alpha$. When $\alpha$ approaches 1, both $N_{\alpha}$ and $N_{\alpha}^*$ decay, and the separability of a point from all other points becomes better.  

\begin{figure}
\centering
a)\includegraphics[width=0.45\textwidth]{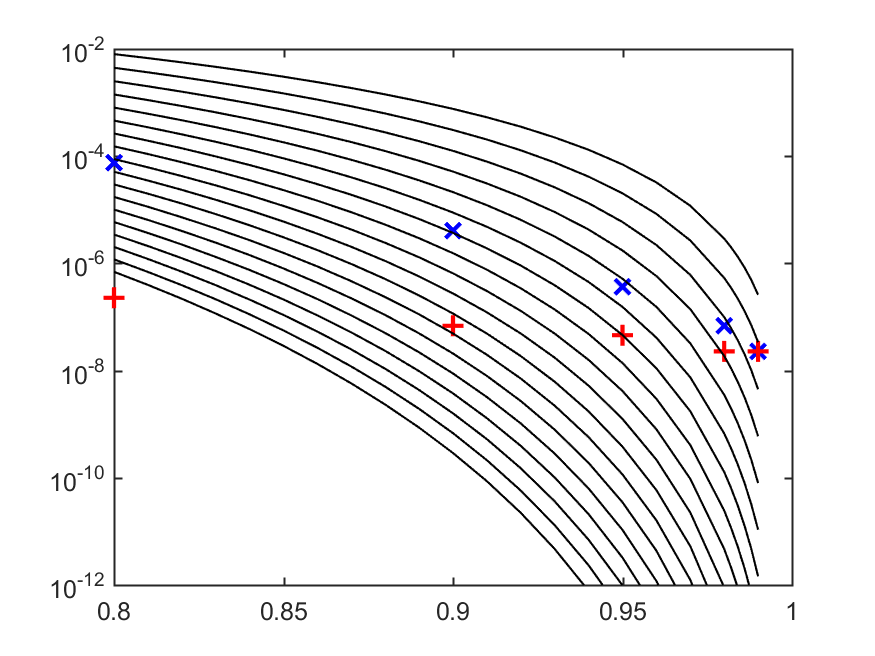} \;\; b)\includegraphics[width=0.45\textwidth]{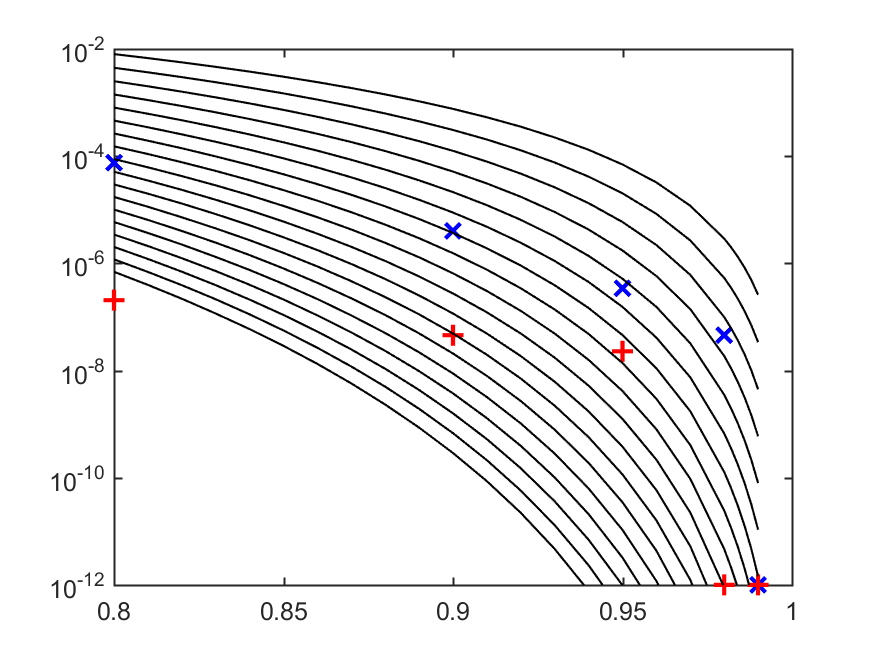}
\caption{Dependence of $\bar{p_y}$ (pluses, ${\mathbf +}$) and  $\bar{p_y}^*$ (skew crosses, ${\mathbf \times}$) on $\alpha$ for the preprocessed LWF data set projected on the unit sphere (data from Table~\ref{tab:separability}). Solid lines -- dependencies of $\bar{p_y}$ on $\alpha$ for equidisrtibutions on the unite sphere  in different dimensions $n$ (\ref{p_y on sphere}, from $n=8$ to $n=25$, from the top down. (a) Original database with saturation at $\alpha=0.99$; (b) The database after fixing two labeling mistakes. Pluses and crosses on the bottom line correspond to zero values.}
\label{Fig:pwithoutlier}
\end{figure}

In the projection on the unit sphere, $N_{\alpha},N_{\alpha}^*,\bar{p_y}, \bar{p_y}^*\to  0$ when $\alpha \to 1$. This is a trivial separability on the border of a strongly convex body where each point is an extreme one. The rate of this tendency to 0 and the separability for $\alpha<1$ depend on the dimension and on the intrinsic structures in data (multicluster structure, etc.). Compare the observed behaviour of $\bar{p_y}$ and $\bar{p_y}^*$ for the LWF data set projected on the unit sphere to $p_y$ for equidistribution on the unit spere (\ref{p_y on sphere}). We can see in Fig.~\ref{Fig:pwithoutlier} (a) that the values of $\bar{p_y}$ for the preprocessed LWF database projected on the unit sphere correspond approximately to the equidistribution on the sphere in dimension 16, and for $\alpha=0.9$ this effective dimension decreases to 14. For $\bar{p_y}^*$ we observe higher effective dimensions: it is approximately 27 for $\alpha=0.8$ and 19 for $\alpha=0.9$.  There is a qualitative difference between behaviour of $\bar{p_y}$ for empirical database and for the equidistribution. For the equidistribution on the sphere, $\bar{p_y}$ decreases approaching the point $\alpha=1$  from below like $const\times (1-\alpha)^{(n-1)/2}$. In the logarithmic coordinates it should look like $const+0.5(n-1)\ln(1-\alpha)$, exactly as we can see in Fig.~\ref{Fig:pwithoutlier}. The behavior of $\bar{p_y}$ for  the LWF database is different and demonstrates some saturation rather than decay to zero. The human inspection of the inseparability at $\alpha=0.99$ shows that there are two pairs of images, which completely coincide but are labeled differently. This is an error in labeling. After fixing of this error, the saturation vanished (Fig.~\ref{Fig:pwithoutlier} (b)) but the obvious difference between curves for the empirical data and the data for equidistribution on the sphere remained. We assume that this difference is caused by multicluster structure of the preprocessed LWF database.

\section{Summary}

\begin{itemize}
\item The problem of correctors  is formulated for AI systems in a multidimensional world. The proposed ideal structure of such correctors consists of two parts: (i) a classifier that separates  situations with diagnosed errors from  situations without diagnosed errors and (ii) a modified decision rule for the situations with high risk of errors.  The technical requirements to the `ideal' correctors are formulated.
\item If a data set is `essentially high-dimensional' then each point can be  separated  with high probability from the rest of the set by a simple Fisher's linear discriminant, even for exponentially large data sets. This is a manifestation of the `blessing of dimensionality' phenomena. In the previous works this statement was proven for equidistributions in a multidimensional ball or for products of measures in a cube (independent components of `Naive Bayes' approximation with bounded support) \cite{GorbTyu2017,GorTyukPhil2018}, but real data rarely have such distribution. The problem of characterising of essentially high-dimensional distributions remained open.
\item In this paper we solve the problem of characterizing `essentially multidimensional distributions'. Roughly speaking, all such distributions are constructed using the following property: the probability of a point to belong to a set of small volume cannot be large (with various specifications  of `small' volume and `large'  probability for this context in spaces of   growing dimension,  see inequality  (\ref{bounded}) and Theorem \ref{Theorem:ExclVol2} as the first examples). We introduced a wide class of distribution with SmAC  property defined through relations between the volume and probability of sets of vanishing volume and proved stochastic separation theorem for these measures (Definition \ref{Def:SmAC}). For SmAC distributions, the general stochastic separation theorem holds. According to this theorem, it is possible to separate every data point from all other points by linear functionals with high probability and even for exponentially large data sets (Theorem \ref{th:separation}).
\item Separation by Fisher's linear discriminants is more attractive for applications than just linear separability. We found a series of conditions for Fisher's separability in high dimensions, from conditions (\ref{bounded}), which can be considered as the reinforced SmAC condition, to log-concavity of distributions.  The areas with relatively low volume but high probability do not exist for all these distributions. The {\em comparison theorem} (Theorem~\ref{prop:domin}) enables us to obtain new stochastic separation conditions by comparing the distributions under consideration with distributions from known separability conditions.
\item Real life data sets and streams of data are not i.i.d samples. The complicated correlations and multi-cluster structures are very common in real life data. We formulated a series of statements about stochastic separation  without i.i.d. hypothesis, for example, Theorems~\ref{Theorem:ExclVol2}, \ref{th:noisy}, and Example \ref{ex:noisy}. This work should be continued to make the theoretical foundations of machine learning closer to real life. 
\item The LFW case study (Sec.~\ref{Sec:test}) demonstrates that (i) Fisher's separability exists in the real life data sets, indeed, (ii) the real life data set  can be significantly different from simple distributions, such as equidistribution on a sphere,  (iii) the separability statistical property of the real data sets might correspond to the separability properties of simple equidistribution on the sphere in lower dimensions (in the example it varies between  14 and 27 instead of original 57 principal components), and (iv) analysis of separability can help in handling errors and outliers.
\item Stochastic separation theorems provide us with the theoretical foundations of fast non-iterative correction of legacy AI systems. These correctors should be constructed without damaging of the skills of the systems in situations, where they are working correctly, and with the ability to correct the new error without destroying the previous fixes. There are some  analogies with the classical cascade correlation \cite{fahlman1990cascade} and greedy approximation \cite{Barron} and with some recent works like neurogenesis deep learning \cite{draelos2016neurogenesis} and deep stochastic configuration networks \cite{Wang2017}. All these methods rely upon new nodes to fix errors. Idea of cascade is also very useful. The corrector technology differs from these prototypes in that it  does not need computationally expensive training or pre-conditioning, and  can be set up as a non-iterative one-off procedure.
\end{itemize}

\section{Conclusion and outlook}

It is necessary to correct  errors of AI systems. The future development of sustainable large AI systems for mining of big data requires creation of technology and methods for fast non-iterative, non-destructive, and reversible corrections of Big Data analytic systems and for fast assimilation of new skills by the networks of AI. This process should exclude human expertise as far as it is possible.

In this paper, we present  a series of theorems, which aims to be a probabilistic foundation of the AI corrector technology. We demonstrate  that classical, simple and robust linear Fisher's discriminant can be used for development of correctors if the data clouds are essentially high-dimensional. We present wide classes of data distributions for which the data sets are linearly separable and Fisher-separable in high dimensions. New stochastic separation theorems demonstrate that the corrector technology can be used to handle errors in data flows with very general probability distributions and far away from the classical i.i.d. hypothesis.

These theorems give the theoretical probabilistic background of the blessing of dimensionality and correction of AI systems by linear discriminants. The cascades of independent linear discriminants are also very simple and even more efficient \cite{GorTyuRom2016,GorbanRomBurtTyu2016}. We have systematically tested linear and cascade correctors with simulated data and on the processing of  real videostream data \cite{GorbanRomBurtTyu2016}. The combination of low-dimensional non-linear decision rules with the  high-dimensional simple linear discriminants is a promising direction of the future development of algorithms.

Combination of legacy AI systems with cascades of simple correctors can modify the paradigm of AI systems. There are {\em core system}. Development of core systems requires time, computational resources and large data collections. These systems are higly integrated. Various special methods are developed for their creation and learning from data including very popular now deep learning \cite{Goodfellow-et-al-2016}. There are also methods for simplification of the core systems, from linear and non-linear principal component analysis \cite{GorbanZinovyev2009} to special algorithms for reduction of deep neural networks \cite{GorbanMirkesTukin2018,Iandola2016}. Core systems could be also rule based systems or hybrid systems which employ  combination of methods from different AI fields.

Functioning of a core system in  realistic operational evironment with concepts drift, non i.i.d. data and other problems requires their regular corrections.  A cloud of `ad hoc' correctors will protect the core system from errors. This protection has two aspects: (i) it is aimed at improving the quality of the functioning of AI and (ii) prolongs its service life in a changing world. Reengineering with the interiorisation of  cloud of correctors skills into the core AI is also possible \cite{GorbanGrechukTyukin2018}. These clouds of correctors help also transfer skills between AI systems \cite{Tyukin2017a}. Mutual corrections in AI networks allow them to adapt to a changing world \cite{GorbanGrechukTyukin2018}. 

This vision of the future augmented artificial intelligence is supported by some  recent achievements of modern neurophysiology.  One of the brain mystery is the well-documented phenomenon of Grandmother Cells and the so-called Concept Cells \cite{QuianQuiroga2012}. Some neurons respond
unexpectedly selective to particular persons or objects. The small ensembles of neurons enable brain to respond selectively to rare  individual stimuli and such selectivity can be learned very rapidly from limited number of experiences.  The system of  small selective ensembles resembles the clouds of correctors. The model based on this  analogy  is capable of explaining: (i) the extreme selectivity of single neurons to the information content, (ii) simultaneous separation of several uncorrelated stimuli or informational items from a large set, and (iii) dynamic learning of new items by associating them with already known ones \cite{TyukinBrain2017}.

We expect that the future development of the paradigm of augmented artificial intelligence based on the hierarchy: core AI systems - clouds of simple correctors - networks of interacting AI systems with mutual corrections will help in solution of the real life problems and in the development of interdisciplinary neuroscience. These expectations are already supported by applications to  several industrial and interdisciplinary research projects. 

\section*{Acknowledgment}

ANG and IYT were Supported by Innovate UK (KTP009890 and KTP010522) and Ministry of science and education, Russia (Project 14.Y26.31.0022). BG  thanks the University of Leicester for granting him academic study leave to do this research.

The proposed corrector methodology was implemented and successfully tested with videostream data and security tasks in collaboration with industrial partners: Apical, ARM, and VMS under support of InnovateUK. We are grateful to them and personally to I. Romanenko, R. Burton, and  K. Sofeikov.

We are grateful to M.~Gromov, who attracted our attention to the seminal question about product distributions in a multidimensional cube, and to G.~Hinton for the important remark that the typical situation with the real data flow is far from an i.i.d. sample (the points we care about seem to be from different distributions).

\end{document}